\documentclass[preprint,12pt,authoryear]{elsarticle}

\usepackage{amssymb}
\usepackage{amsmath}
\usepackage{amsthm}
\usepackage{enumerate}
\usepackage{graphicx}
\usepackage{subfigure}
\usepackage{setspace}
\usepackage{mathrsfs}
\usepackage{color}
\usepackage{multirow}
\usepackage{array}
\usepackage{pdfpages}
\usepackage{parskip}
\usepackage{tabularx}
\usepackage{arydshln}
\usepackage{booktabs}
\usepackage{threeparttable}
\usepackage{placeins}
\usepackage{makecell}
\usepackage{longtable}
\usepackage[cmyk]{xcolor}

\theoremstyle{thmstyleone}%
\newtheorem{theorem}{Theorem}
\newtheorem{proposition}{Proposition}

\theoremstyle{thmstyletwo}%

\theoremstyle{thmstylethree}%
\newtheorem{definition}{Definition}%

\newcolumntype{C}[1]{>{\centering\arraybackslash}m{#1}}

\journal{Neural Networks}

\begin{document}

\begin{frontmatter}



\title{Minimum Block Width for Universal Approximation by Residual Neural Networks with Inner Width One} 


\author[aff1,aff2]{Qi Zhou}
\ead{qizhou1037@hust.edu.cn}

\author[aff1,aff2]{Xuan Zhou}
\ead{xuanzhou1037@hust.edu.cn}

\author[aff1,aff2]{Xiao-Song Yang\corref{cor1}}
\ead{yangxs@hust.edu.cn}

\cortext[cor1]{Corresponding author}
 
\affiliation[aff1]{organization={School of Mathematics and Statistics},
            addressline={Huazhong University of Science and Technology}, 
            city={Wuhan},
            postcode={430074}, 
            state={Hubei},
            country={P.R. China}}

\affiliation[aff2]{organization={Hubei Key Laboratory of Engineering Modeling and Scientific Computing},
	addressline={Huazhong University of Science and Technology}, 
	city={Wuhan},
	postcode={430074}, 
	state={Hubei},
	country={P.R. China}}
\begin{abstract}
In this paper, we study the universal approximation property of residual neural networks. For input and output dimensions $d_x$ and $d_y$, and LeakyReLU, ReLU, ReLU-like activation functions, the upper and lower bounds of the minimum block width are established. To achieve $L^p$ approximation $(1\leq p <+\infty)$ on any compact set, we show that the exact minimum block width is $\max\{d_x,d_y\}$ when each residual branch has inner width 1. Furthermore, we show that residual neural networks with block width $\min\{d_x+d_y, \max\{2d_x+1,d_y\}\}$ can achieve uniform approximation on any compact set under the constraint that each residual branch has inner width 1. Besides, for any activation function family, we prove that there exist functions that cannot be approximated by residual neural networks with block width less than $\max\{d_x, d_y\}$, both in the $L^p$ sense and the uniform sense, regardless of inner width. Consequently, for LeakyReLU, ReLU, ReLU-like activation functions and $d_y\geq 2d_x+1$, the exact minimum block width for uniform approximation is $d_y$ when each residual branch has inner width 1.
\end{abstract}




\end{frontmatter}

\section{Introduction}
\subsection{Motivation}
Residual Neural Networks (ResNets) play an important role in machine learning tasks (e.g., image recognition \citep{ResNet_HeKaiming}, object detection \citep{Object_detection}). Thus, the approximation theory of ResNets has become an important topic in the mathematical theory of deep learning. 

Mathematically, a ResNet can be viewed as a composition of residual blocks. Each residual block consists of an identity map and a residual branch, where the residual branch is itself a small neural network. The outputs of these two components are combined by vector addition. Thus, each residual block acts on $\mathbb{R}^w$ according to $\mathcal B(x)=x+\Phi(x)$, where the identity map provides the shortcut connection and $\Phi$ denotes the residual branch. In particular, ResNets have been interpreted as time discretizations of ordinary differential equation (ODE) flows on $\mathbb{R}^{d_x}$ \citep{EWeinan}, and it has been proved that ResNets can approximate every function which can be connected to an identity function through an analytic and monotone homotopy \citep{ResNet_geo_control}. However, recent studies have indicated that flow-based residual transformations may be limited by topological obstructions of the input set \citep{AugmentODE}. As a result, ODE-induced flows may fail to approximate arbitrary continuous maps from $\mathbb{R}^{d_x}$ to itself because of topological obstructions. This motivates the augmented ResNet architecture, which first lifts the input from $\mathbb{R}^{d_x}$ into a higher dimensional ambient space $\mathbb{R}^w$ by an initial linear layer, then performs residual transformations in the enlarged space, and finally projects the transformed input set from $\mathbb{R}^w $ onto the desired output space $\mathbb{R}^{d_y}$ through the final linear layer. Due to the benefits of augmented ResNets \citep{Augment_i_ResNet}, we adopt this augmented ResNet architecture throughout this paper.

In contrast to the width notion for Multi-Layer Perceptrons (MLPs), augmented ResNets have two width parameters. On the one hand, each residual block acts on the ambient space $\mathbb{R}^w$, then we call this dimension the block width $w$. On the other hand, the residual branch inside each block is itself a small neural network, and the width of the residual branch is called the inner width $\bar{w}$ (see Definition \ref{def_ResNetWidth}). These two quantities measure different sources of expressive power. The block width determines the dimension of the ambient space in which the residual blocks act, whereas the inner width measures the complexity of each local nonlinear update.

From the viewpoint of approximation theory, it is natural to ask how large the block width and inner width need to be in order to approximate a prescribed function class. This paper focuses on the extreme case that every residual branch has inner width one, and it leads to the following questions:
\begin{itemize}
	\item Can ResNets with an extremely small inner \textbf{}width $\bar{w}$, even inner width one, still retain universal approximation property if the block width $w$ is chosen appropriately?
	
	\item When the inner width $\bar{w}$ is fixed as 1, how large the block width $w$ must be for $L^p$ approximation of functions in $L^p(K, \mathbb{R}^{d_y})$, and for uniform approximation of functions in $C(K, \mathbb{R}^{d_y})$, respectively, on any compact set $K\subset \mathbb{R}^{d_x}$?
	
	\item Is there a lower bound, such that ResNets with block width below this lower bound fail to achieve $L^p$ approximation of functions in $L^p(K, \mathbb{R}^{d_y})$ and uniform approximation of functions in $C(K, \mathbb{R}^{d_y})$, on any compact set $K\subset\mathbb{R}^{d_x}$?
	
\end{itemize}
Then the purpose of this paper is to answer these questions.

\subsection{Related Work}
Universal Approximation Property (UAP) and minimum width theory for MLPs have been studied in recent years \citep{Hanin, Johnson, Kidger, Park, Cai, Hwang_C0, Kim}. However, the study of ResNets differs from that for MLPs, because a residual block contains both the block width and the inner width of the nonlinear residual branch.

In this line of research, the minimum block width and inner width of ResNets have been studied in several recent works. \citet{ResNet_one-neuron} proved that for maps in $L^1(\mathbb{R}^{d_x}, \mathbb{R})$, ReLU ResNets with a single neuron per hidden layer, i.e., when the inner width is 1, suffice to achieve the UAP. \citet{DL_via_DS} proved that ODENets, which are regarded as the continuous-time counterparts of ResNets, can achieve the $L^p$ approximation of functions in $L^p(K, \mathbb{R}^{d_y})\enspace (p\in [1, +\infty))$ for any compact set $K\subset\mathbb{R}^{d_x}$, $d_x\geq d_y$, and activation functions such as ReLU, Sigmoid, Tanh. \citet{ODEnet_DS} proved that ResNets and ODENets with block width $d_x+d_y$ can uniformly approximate all functions in $C(K, \mathbb{R}^{d_y})$ for any compact set $K\subset\mathbb{R}^{d_x}$, when $d_x\geq d_y$ and the activation function is non-polynomial. Besides, \citet{ResNet_geo_control} established a universal approximation result that ResNets with block width $2\max\{d_x, d_y\}+1$ can uniformly approximate all functions in $C(K, \mathbb{R}^{d_y})$ for any compact set $K\subset \mathbb{R}^{d_x}$, under the assumption that the activation function is monotonic, and one of its derivatives satisfies a quadratic differential equation.

Despite these developments, the minimum block width of ResNets under inner width constraints remains less understood. In particular, it is natural to ask whether the residual branches can be reduced to inner width one while preserving the universal approximation property, and to determine that for a given target map, which block width is sufficient or necessary for $L^p$ approximation and uniform approximation on any compact set.

\subsection{Contributions of this Paper}
Our main contributions can be summarized as follows. For the reader's convenience, we summarize the existing results of ResNets and ODENets, and our contributions in Table 1. 
\begin{itemize}
	\item We focus on the case that inner width is 1 because it represents the most restrictive architecture for the residual branches. In this paper, we show that the inner width of residual branches can be reduced to one without losing the UAP, as long as the block width is chosen appropriately. Our result (Theorem \ref{thm_minwidth_Lp}) extends the output space $\mathbb{R}^{d_y}$ in \citep{ResNet_one-neuron} for $d_y>1$. Besides, it extends the range of activation functions from $\mathrm{ReLU}$ to $\mathrm{LeakyReLU}$ and ReLU-like activation functions (including ELU, CELU, SELU, Softplus, GELU, SiLU, Mish, ReLU6, Softshrink, HardSigmoid, HardTanh and HardSwish).
	
	\item We show that for ResNets with LeakyReLU, ReLU or ReLU-like activation functions, to achieve the $L^p$ approximation on any compact set, the minimum block width is $w_{\min}=\max\{d_x, d_y\}$ (Theorem \ref{thm_minwidth_Lp}). Since the class with inner width one is contained in the class with larger inner width, the same upper bounds also hold for every $\bar w\ge 1$. This result removes the dimensional restriction $d_x \geq d_y$ in \citep{DL_via_DS}.
	
	\item We show that for ResNets with LeakyReLU, ReLU or ReLU-like activation functions, to achieve the uniform approximation of continuous functions on any compact set, the minimum block width satisfies $\max\{d_x, d_y\}\leq w_{\min}\leq\min\{d_x+d_y, \max\{2d_x+1,d_y\}\}$ (Theorem \ref{thm_minwidth_C}). The same block-width result holds for every inner width $\bar w\ge1$. Our upper bound is strictly smaller than the result in \citep{ResNet_geo_control} for all $d_x , d_y\in \mathbb{N}_+$, and our activation functions are not required to be monotonic. Furthermore, compared with the result in \citep{ODEnet_DS}, our result recovers the upper bound $d_x+d_y$ in their regime $d_x\geq d_y$ and extends this result to arbitrary input and output dimensions. Moreover, when $d_y\geq d_x+2$, our upper bound is strictly smaller than $d_x+d_y$.
	
	\item We show that ResNets with block width less than $\max\{d_x, d_y\}$ cannot approximate all target functions on any compact set, both in the $L^p$ sense and the uniform sense, regardless of the activation family, and the same lower bound holds for every inner width $\bar{w}\geq 1$. Consequently, for activation functions LeakyReLU, ReLU or ReLU-like activation functions, when $d_y\geq 2d_x+1$, the exact minimum block width for uniform approximation is $d_y$.
\end{itemize}

\subsection{Organization}
The remainder of this paper is organized as follows. Section 2 introduces the notations and definitions throughout this paper. Section 3 proves the fundamental approximation properties of compositions of residual blocks with inner width one, including the approximation of affine transformations and coordinate-wise nonlinear maps. Section 4 establishes the minimum block width bounds for $L^p$ approximation and uniform approximation. Finally, Section 5 discusses several possible directions for future work.

\begin{table}[!htbp]
	\caption{Comparison of width requirements for ResNets and ODENets. 
		$K$ denotes a compact set in $\mathbb{R}^{d_x}$, and the output space is $\mathbb{R}^{d_y}$.}
	\label{table_summary_results}
	\scriptsize
	\centering
	\setlength{\tabcolsep}{3pt}
	\renewcommand{\arraystretch}{1.45}
	\begin{tabular}{
			@{} 
			C{0.13\textwidth}
			C{0.14\textwidth}
			C{0.18\textwidth}
			C{0.30\textwidth}
			C{0.18\textwidth}
			@{}}
		\toprule
		\multicolumn{5}{c}{\textbf{Existing Results of ResNets and ODENets}}\\
		\midrule
		\textbf{References} 
		& \textbf{Domain} 
		& \textbf{Activation} 
		& \textbf{Block Width} 
		& \textbf{Inner Width} \\
		\midrule
		
		\cite{ResNet_one-neuron}
		&
		$L^1(K,\mathbb{R})$
		&
		ReLU
		&
		$w_{\mathrm{min}}\leq d_x$\textsuperscript{(2)}
		&
		$\bar{w}=1$\\
		\midrule
		
		\cite{DL_via_DS}
		&
		$L^p(K,\mathbb{R}^{d_y})$\newline
		($d_x\geq d_y$)
		&
		ReLU,\newline
		Sigmoid,\newline
		Tanh, etc.
		&
		$w_{\mathrm{min}}\leq d_x$\textsuperscript{(2)}
		&
		not discussed
		\\
		\midrule
		
		\cite{ODEnet_DS}
		&
		$C(K,\mathbb{R}^{d_y})$\newline
		$(d_x\geq d_y)$
		&
		Non-polynomial
		&
		$w_{\mathrm{min}}\leq d_x+d_y$\textsuperscript{(2)}
		&
		not discussed
		\\
		\midrule

		\cite{ResNet_geo_control}
		&
		$C(K,\mathbb{R}^{d_y})$
		&
		Lipschitz\newline
		continuous\textsuperscript{(1)}
		&
		$w_{\mathrm{min}}\leq 2\max\{d_x,d_y\}+1$\textsuperscript{(2)}
		&
		not discussed
		\\
		\bottomrule
		\addlinespace[1.5em]
		\toprule
		\multicolumn{5}{c}{\textbf{Ours}}\\
		\midrule
		\textbf{References} 
		& \textbf{Domain} 
		& \textbf{Activation} 
		& \textbf{Block Width} 
		& \textbf{Inner width} \\
		\midrule
		
		Ours\newline
		(Theorem~\ref{thm_minwidth_Lp})
		&
		$L^p(K,\mathbb{R}^{d_y})$\newline
		&
		LeakyReLU,\newline
		ReLU,\newline
		ReLU-like\textsuperscript{(3)}
		&
		$w_{\mathrm{min}}=\max\{d_x,d_y\}$
		&
		$\bar{w}=1$
		\\
		\midrule
		
		Ours\newline
		(Theorem~\ref{thm_minwidth_C})
		&
		$C(K,\mathbb{R}^{d_y})$
		&
		LeakyReLU,\newline
		ReLU,\newline
		ReLU-like\textsuperscript{(3)}
		&
		$\max\{d_x,d_y\}\leq w_{\mathrm{min}}\leq$\newline
		$\min\{d_x+d_y,\max\{$ \newline $2d_x+1,d_y\}\}$
		&
		$\bar{w}=1$
		\\
	    \midrule
		
		Ours\newline
		(Theorem~\ref{thm_minwidth_C})
		&
		$C(K,\mathbb{R}^{d_y})$\newline
		($d_y\geq 2d_x+1$)
		&
		LeakyReLU,\newline
		ReLU,\newline
		ReLU-like\textsuperscript{(3)}
		&
		$w_{\mathrm{min}}=d_y$
		&
		$\bar{w}=1$
		\\
		\bottomrule
	\end{tabular}
	
	\vspace{0.4em}
	
	\begin{minipage}{0.98\textwidth}
		\scriptsize
		\textsuperscript{(1)} Activation function is Lipschitz, monotonic, and one of its derivatives satisfies a quadratic differential equation.\newline
		\textsuperscript{(2)}  They only provide the sufficient upper bound, without discussing the lower bound.\newline
		\textsuperscript{(3)} ReLU-like functions include ELU, CELU, SELU, Softplus, GELU, SiLU, Mish, ReLU6, Softshrink, HardSigmoid, HardTanh and HardSwish.
	\end{minipage}
\end{table}

\FloatBarrier
\section{Preliminaries}
In this section, we introduce the mathematical notation throughout this paper and give the formal mathematical definition about the ResNet.

\subsection{Notations and Definitions}
We first introduce the notations and definitions throughout this paper.
\begin{itemize}
	\item \textbf{Column Vector}: $x=(x_1, x_2, \cdots, x_n)$ denotes an $n$ dimensional column vector unless otherwise stated, where $x_i$ is the $i$-th component of $x$. For convenience, we identify $R^{k}$ as the $k$ dimensional column vector space $R^{k\times 1}$.
	
	\item \textbf{$\mathbf{sup}$ norm and $L^p$ norm of $f$}: Let $\mu$ be the Lebesgue measure on $\mathbb{R}^n$, for a measurable mapping $f:\mathbb{R}^n \to \mathbb{R}^m$, for any compact measurable set $A\subset\mathbb{R}^n$, we can define the $\mathrm{sup}$ norm of $f$ on $A$ by:
	\begin{align*}
		\vert\vert f \vert\vert_{\mathrm{sup},A} = \sup_{x\in A}\vert\vert f(x)\vert\vert
	\end{align*}
	where $\vert\vert * \vert\vert$ is the Euclidean norm of $\mathbb{R}^m$. For $1\leq p <+\infty$, we can also define the $L^p$ norm of $f$ on $A$ by
	\begin{align*}
		\vert\vert f \vert\vert_{p,A} = \left( \int_{A} \vert\vert f(x) \vert\vert^p d\mu\right)^{\frac{1}{p}}.
	\end{align*}
	
	\item \textbf{Local $L^p$ Space}: For measurable sets $X\subset\mathbb{R}^{d_x}, Y\subset \mathbb{R}^{d_y}$. We denote the local $L^p$ space by
	\begin{align*}
		L^p_{\mathrm{loc}}(X, Y):=\{f\vert \text{$f:X\to Y$ is measurable}, \forall\text{compact} K, \vert\vert f\vert\vert_{p, K}<\infty\}.
	\end{align*}
	It is easy to verify that $C(X, Y)\subset L^p_{\mathrm{loc}}(X, Y)$. Besides, when $X$ is a compact set, $L^p_{\mathrm{loc}}(X, Y)=L^p(X, Y)$.
	
	\item \textbf{Compactly Approximate}: For measurable sets $X\subset\mathbb{R}^n$, $Y\subset\mathbb{R}^m$, and function classes $\mathcal{F}\subset C(X, Y), \mathcal{G}\subset C(X, Y)$, we write $\mathcal{G} \underset{\mathrm{sup}}{\prec} \mathcal{F}$ if for any $g \in \mathcal{G}$, $\varepsilon >0$ and any compact set $K\subset X$, there exists $f\in \mathcal{F}$ such that $\vert\vert f -g \vert\vert_{\mathrm{sup},K} <\varepsilon$. For $\mathcal{F}\subset L^p_{\mathrm{loc}}(X, Y), \mathcal{G}\subset L^p_{\mathrm{loc}}(X, Y)$, we define $\mathcal{G}\underset{p}{\prec} \mathcal{F}$ if for any $g\in \mathcal{G}$,$\varepsilon>0$ and any compact set $K\subset X$, there exists $f\in \mathcal{F}$ such that $\vert\vert f-g\vert\vert_{p, K}< \varepsilon$. For notational convenience, we use the notations $g \underset{\mathrm{sup}}{\prec} \mathcal{F}$, $g \underset{p}{\prec} \mathcal{F}$ to denote $\{g\}\underset{\mathrm{sup}}{\prec} \mathcal{F}$ and $\{g\}\underset{p}{\prec} \mathcal{F}$, respectively.
	
	By the triangle inequality for the sup norm and $L^p$ norm, it is easy to verify the transition property:
	\begin{itemize}
	 	\item[(a)] For $\mathcal{F}, \mathcal{G}, \mathcal{H}\subset C(X, Y)$, if $\mathcal{G}\underset{\mathrm{sup}}{\prec}\mathcal{F}$ and $\mathcal{H}\underset{\mathrm{sup}}{\prec}\mathcal{G}$, then $\mathcal{H}\underset{\mathrm{sup}}{\prec}\mathcal{F}$.
	 	
	 	\item[(b)] For $\mathcal{F}, \mathcal{G}, \mathcal{H}\subset L^p_{\mathrm{loc}}(X, Y)$, if $\mathcal{G}\underset{p}{\prec}\mathcal{F}$ and $\mathcal{H}\underset{p}{\prec}\mathcal{G}$, then $\mathcal{H}\underset{p}{\prec}\mathcal{F}$.
	 \end{itemize}
		
	\item \textbf{Composition of Function Classes}: For function classes $\mathcal{F}_i \enspace (i\in \{1, 2, \cdots, N\})$, then we define the composition of function classes $\mathcal{F}_N\circ\cdots\circ \mathcal{F}_1$ by
	\begin{align*}
		\mathcal{F}_N\circ\cdots\circ \mathcal{F}_1:=\{f_N\circ \cdots\circ f_2\circ f_1\vert f_i\in \mathcal{F}_i, i\in \{1, 2, \cdots, N\}\}.
	\end{align*}
	
	\item \textbf{Canonical Coordinate Map}: For $n, m\in \mathbb{N}_+$, we define the canonical coordinate map $\pi_{n, m}:\mathbb{R}^n\to \mathbb{R}^m$ as follows:
	$$
	\pi_{n, m}(x_1, \cdots, x_n)=\left\{
	\begin{aligned}
		&(x_1, \cdots, x_m), & n\geq m,\\
		&(x_1, \cdots, x_n, 0, \cdots, 0), & n< m.
	\end{aligned}
	\right.
	$$
\end{itemize}
Throughout this paper, $I_n$ denotes the $n\times n$ identity matrix, and $E_{n\times m,i,j}$ denotes the elementary matrix whose $(i,j)$-entry is one. We denote $\mathbf{1}_{n\times m}$ and $\mathbf{0}_{n\times m}$ as the $n\times m$ matrix whose entries are all 1 and all 0, respectively. We also denote the $n\times m$ matrix whose diagonal $(i, i)$-entries are all 1 by $\mathrm{Diag}_{n, m}$. We denote $\mathrm{Id}_n:\mathbb{R}^n\to \mathbb{R}^n$ as the $n$ dimensional identity map $\mathrm{Id}(x_1, x_2, \cdots, x_n)= (x_1, x_2, \cdots, x_n)$.

For the set of affine transformations, we write
\begin{align*}
	T_{W, B}(x)&:=Wx+B,\enspace \mathrm{Aff}_{n,m}:=\{T_{W,B}\vert W\in\mathbb R^{m\times n},\ B\in\mathbb R^m\},\\
	\mathrm{IAff}_{n, m}&:=\{T_{W, B}\vert W\in\mathbb{R}^{m\times n}, \mathrm{rank}(W)=\min\{n, m\}, B\in \mathbb{R}^m\}.
\end{align*}

We formally introduce our definition of activation functions.
\begin{definition}[Activation Function]\label{def_activationfunc}
	Activation function $\sigma_0:\mathbb{R}\to \mathbb{R}$ is a piecewise $C^1$ function that possesses at least one point $\alpha$ such that $\sigma_0^\prime(\alpha) \neq 0$. For notational convenience, it can be applied to vector valued functions  as componentwise operators $\mathbb{R}^n \to \mathbb{R}^n$:
	\begin{align*}
		\sigma_0(x_1, \cdots, x_n):= (\sigma_0(x_1), \cdots, \sigma_0(x_n)).
	\end{align*}
\end{definition}
Several common activation functions are listed in Appendix \ref{appendix_activationfunc}. Some families of activation functions, denoted by $\sigma_\beta$, depend on an additional parameter $\beta\in \Lambda$, nevertheless, we regard them as belonging to the same activation class, such as LeakyReLU and ELU. We denote such a collection of activation functions by $\sigma:= \{\sigma_\beta\vert \beta\in \Lambda\}\enspace(\Lambda=\mathbb{N}, \mathbb{R}, \cdots)$. For example, we denote the ELU and LeakyReLU families by 
\begin{align*}
	\mathrm{ELU}:= \{ \mathrm{ELU}_\beta\vert \beta\in \mathbb{R}_{+}\},\enspace \mathrm{LeakyReLU}:=\{\mathrm{LeakyReLU}_\beta\vert \beta\in (0, 1)\cup (1, +\infty)\}.
\end{align*}

\subsection{Preliminaries on Diffeomorphisms}
Before our formal discussion on the ResNet, we first introduce several classes of diffeomorphisms. The definitions are mainly adapted from \citep{Coupling_Flow_NeuralPS, Hwang_C0}.
\begin{definition}[Set of $C^r$-Diffeomorphisms]
	Let $U\subset \mathbb{R}^m$ be an open set, we denote $\mathcal{D}^r(U)$ as the set of all $C^r$-diffeomorphisms $f: U\to f(U)$.  
\end{definition}

\begin{definition}[Invertible Neural Network]
	For $d\in \mathbb{N}_+$, let $\mathcal{G}$ be a set of invertible functions from $\mathbb{R}^d$ to itself. Then the set of invertible neural networks $\mathrm{INN}_\mathcal{G}$ is defined by
	\begin{align}
		\mathrm{INN}_\mathcal{G}:=\bigcup_{n\in \mathbb{N}_+} \{ T_n\circ g_{n-1}\circ \cdots \circ g_2\circ T_2 \circ g_1\circ T_1\vert  g_i\in \mathcal{G}, T_i\in \mathrm{IAff}_{d, d} \}.
	\end{align}
\end{definition}

\begin{definition}[Compactly Supported Diffeomorphism]
	A diffeomorphism $f:\mathbb{R}^d\to \mathbb{R}^d$ is called compactly supported, if there exists a compact subset $K\subset \mathbb{R}^d$ such that $f(x)=x$ for all $x\notin K$. We denote the set of all compactly supported $C^r$-diffeomorphisms of $\mathbb{R}^d$ by $\mathrm{Diff}_c^r(\mathbb{R}^d)$.
\end{definition}

\begin{definition}[Single Coordinate Transformation]
	We denote the set of all compactly supported $C^r$-diffeomorphisms that alter only the last coordinate:
	\begin{align}
		\mathcal{S}_c^r(\mathbb{R}^d):=\{f\in \mathrm{Diff}_c^r(\mathbb{R}^d)\vert f(x)=(x_1, \cdots, x_{d-1}, f_d(x)), \enspace f_d\in C^r(\mathbb{R}^d, \mathbb{R})\}.
	\end{align}
\end{definition}

\subsection{Preliminaries of Residual Neural Networks}
In this subsection, we introduce the definition of residual block and ResNet, and we define the block width, inner width, depth of a ResNet.
\begin{definition}[Residual Block]\label{def_residualblock}
	For a given activation function $\sigma_0$, a block width-$w$ residual block $\mathcal{B}:\mathbb{R}^w\to \mathbb{R}^w$ is defined by
	\begin{align}\label{eq_block}
		\mathcal{B}=\mathrm{Id}_w+\Phi,
	\end{align}
	where the residual branch $\Phi$ is given by
	\begin{align}\label{eq_residual_branch}
		\Phi=T_{W_b, B_b}\circ \sigma_0 \circ T_{W_a, B_a}, \enspace T_{W_a, B_a}\in \mathrm{Aff}_{w, \bar{w}}, \enspace T_{W_b, B_b}\in \mathrm{Aff}_{\bar{w}, w}.
	\end{align}
\end{definition}
For a family of activation functions $\sigma=\{\sigma_i\vert i\in \mathbb{N}_+\}$, the set of finite compositions of residual blocks is defined by
\begin{align}\label{eq_CRB}
	\mathrm{CRB}_{w, \bar{w}}^{\sigma}:=\bigcup_{N\in \mathbb{N}_+}\{\phi:\mathbb{R}^w\to \mathbb{R}^w\vert \phi=\mathcal{B}_{N}\circ \cdots \circ \mathcal{B}_{2}\circ \mathcal{B}_{1}\}.
\end{align}
where each residual block $\{\mathcal{B}_i\vert i\in \mathbb{N}_+\}$ has the form
\begin{align*}
	\mathcal{B}_i=\mathrm{Id}_w+T_{W_{ib}, B_{ib}}\circ \sigma_i \circ T_{W_{ia}, B_{ia}},\enspace T_{W_{ia}, B_{ia}}\in \mathrm{Aff}_{w, \bar{w}}, T_{W_{ib}, B_{ib}}\in \mathrm{Aff}_{\bar{w}, w}, \sigma_i\in \sigma.
\end{align*}
Using the above definition of residual blocks, we now formally define ResNet throughout this paper.
\begin{definition}[ResNet]\label{def_ResNet}
	For a family of activation functions $\sigma=\{\sigma_i\vert i\in \mathbb{N}_+\}$, a ResNet is a map $R:\mathbb{R}^n\to \mathbb{R}^m$ of the form
	\begin{align}\label{eq_ResNet}
		R=T_{W_\beta, B_\beta} \circ \phi \circ T_{W_\alpha, B_\alpha},
	\end{align}
	where $\phi\in \mathrm{CRB}_{w, \bar{w}}^{\sigma}$ and $T_{W_\alpha, B_\alpha}\in \mathrm{Aff}_{n, w}, T_{W_\beta, B_\beta}\in \mathrm{Aff}_{w, m}$. A schematic diagram of this architecture is shown in Figure \ref{fig_ResNet}.
\end{definition}

Subsequently, we define the block width, inner width, and depth of a ResNet.
\begin{definition}[Block Width, Inner Width and Depth of a ResNet]\label{def_ResNetWidth}
	In Equation \eqref{eq_ResNet}, the inner width in each residual block is fixed as $\bar{w}$. We refer to $w$ and $\bar{w}$ in Equation \eqref{eq_ResNet} as the block width and inner width, respectively. The depth of a ResNet is defined as the number of residual blocks it contains.
\end{definition}
We now define the set of ResNets with block width $w$ and inner width $\bar{w}$.
\begin{definition}[Set of ResNets]
	When all activation functions used in the residual blocks belong to the set $\sigma=\{\sigma_i\vert i\in \Lambda\}$, we denote the set of ResNets $R:\mathbb{R}^n\to \mathbb{R}^m$, whose residual blocks all have block width $w$ and inner width $\bar{w}$, by $\mathrm{Res}_{n, m, w, \bar{w}}^{\sigma}$. 
\end{definition}
\begin{figure}[htbp]
	\centering
	\includegraphics[height=0.54\linewidth, width=0.9\linewidth]{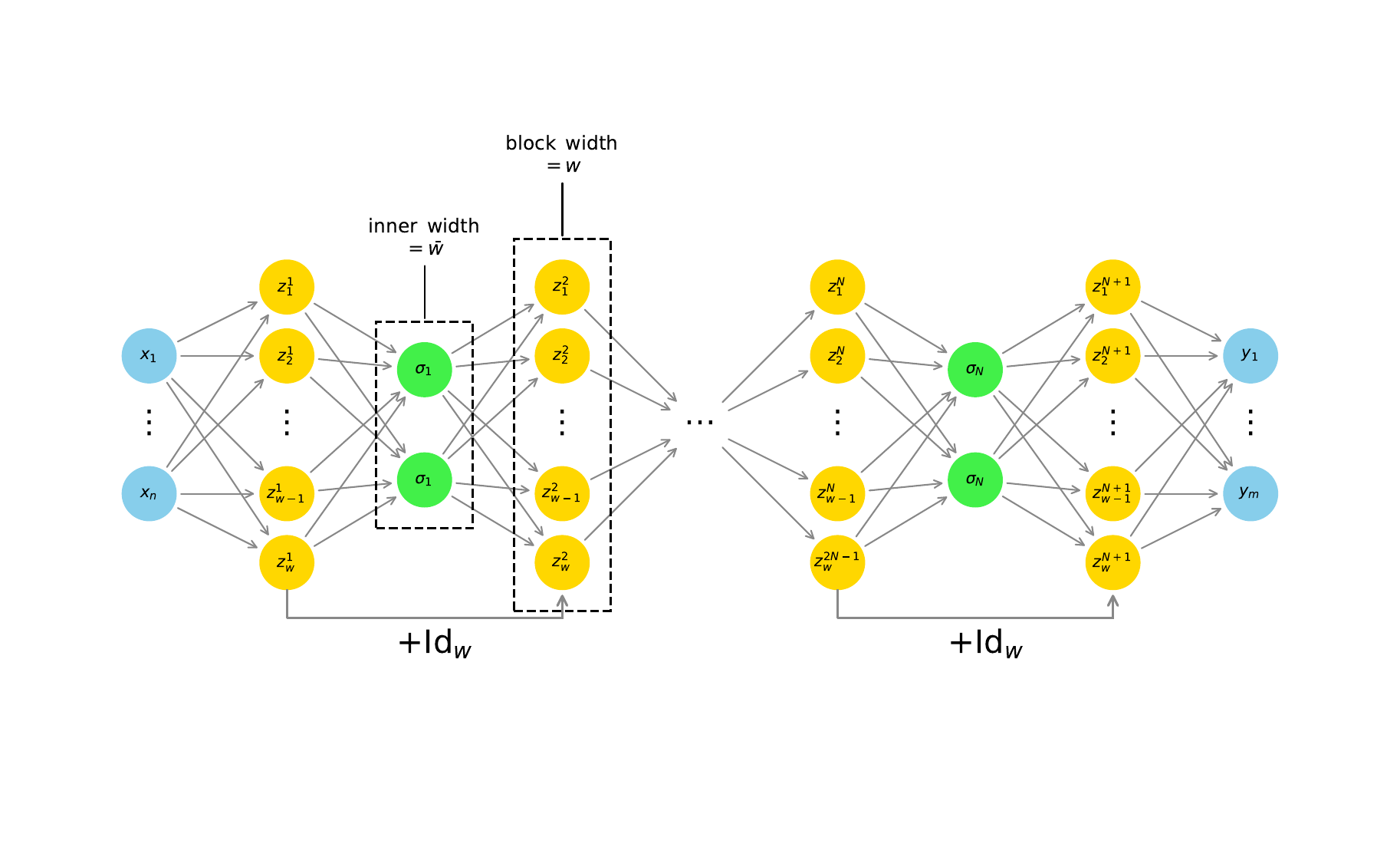}
	\caption{Diagram of a ResNet with block width $w$ and inner width $\bar{w}$. For each input $x=(x_1, x_2, \cdots, x_n)$, the output of this ResNet is $y=(y_1, y_2, \cdots, y_m)$. Each residual block transforms the hidden input $z^i=(z^i_1, \cdots, z^i_w)$ to the hidden output $z^{i+1}=(z^{i+1}_1, \cdots, z^{i+1}_w)$ by the map $\mathcal{B}_i = \mathrm{Id}_w+T_{W_{ib}, B_{ib}} \circ \sigma_i \circ T_{W_{ia}, B_{ia}}$.} 
	\label{fig_ResNet}
\end{figure}

\section{Approximation Properties of Compositions of Residual Blocks}
Since the relation $\mathrm{Res}_{n, m, w, \bar{w}_1}^{\sigma}\underset{\mathrm{sup}}{\prec} \mathrm{Res}_{n, m, w, \bar{w}_2}^{\sigma}$ holds for all $0<\bar{w}_1<\bar{w}_2$ and activation functions $\sigma=\{\sigma_i\vert i\in \Lambda\}$, we focus on the approximation properties of $\mathrm{Res}_{n, m, w, 1}^{\sigma}$, i.e., the class of ResNets whose inner width is fixed as 1 in this paper.

In this section, we discuss the approximation properties of $\mathrm{Res}_{n, m, w, 1}^{\sigma}$. A residual block with block width $w$ and inner width 1 is illustrated in Figure \ref{fig_block_inner_1}.
\begin{figure}[htbp]
	\centering
	\includegraphics[height=0.5\linewidth, width=0.5\linewidth]{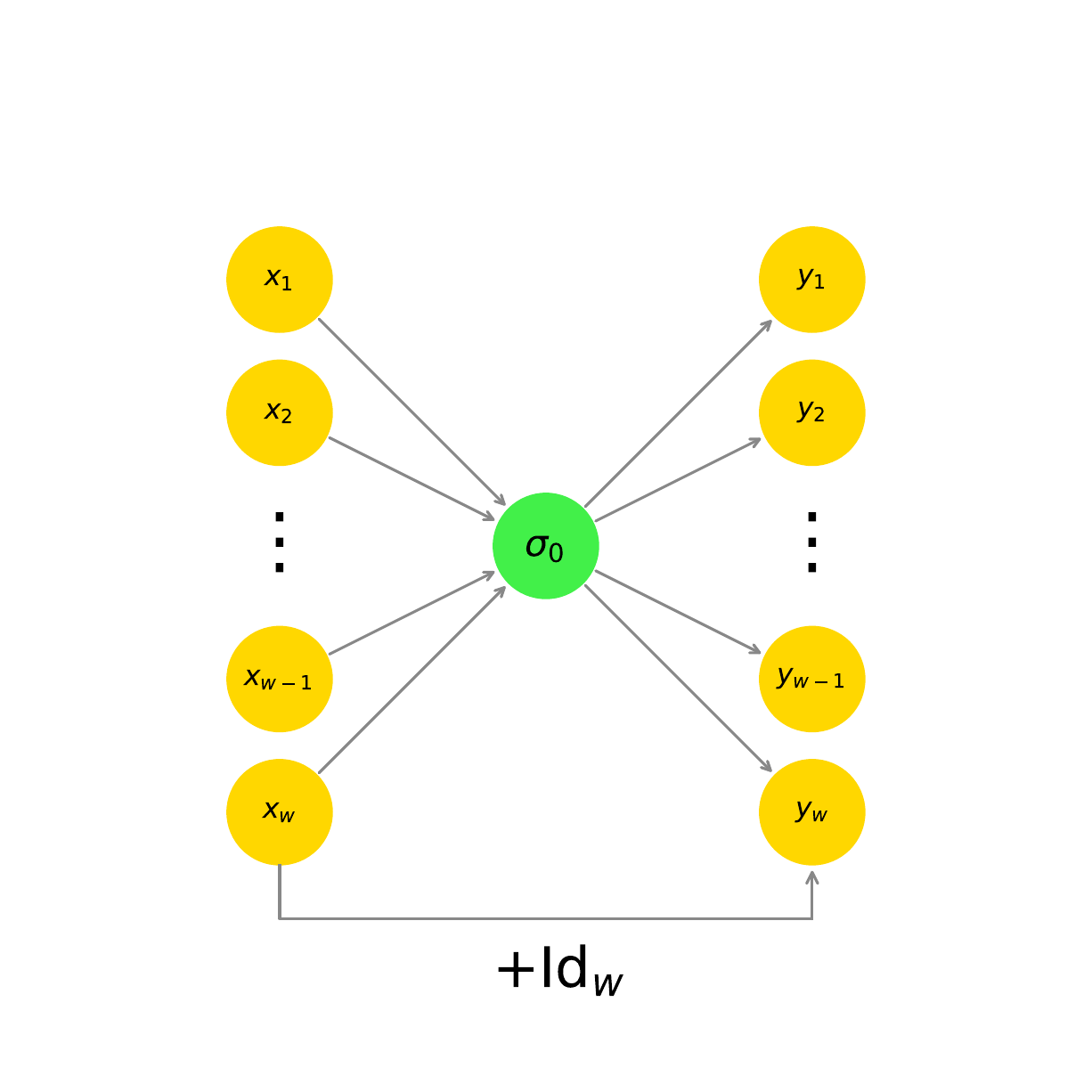}
	\caption{Diagram of a residual block with activation function $\sigma_0$, block width $w$ and inner width 1. This block transforms the input $x=(x_1, x_2, \cdots, x_w)$ to the output $y=(y_1, y_2, \cdots, y_w)$ by the map $\mathcal{B}=\mathrm{Id}_w+T_{W_b, B_b}\circ \sigma_0 \circ T_{W_a, B_a}$.}
	\label{fig_block_inner_1}
\end{figure}

\subsection{The Ability of Compactly Approximating Affine Transformations}
In this subsection, we prove that for any set of activation functions $\sigma=\{\sigma_i\vert i\in \Lambda\}$, compositions of residual blocks can compactly approximate every affine transformation from $\mathbb{R}^w$ to itself on any compact set, both in the $\mathrm{sup}$ norm sense and the $L^p$ norm sense, i.e., 
\begin{align*}
	T_{W, B}\underset{\mathrm{sup}}{\prec} \mathrm{CRB}^{\sigma}_{w, 1}, \enspace T_{W, B}\underset{p}{\prec} \mathrm{CRB}^{\sigma}_{w, 1},\enspace \forall T_{W, B}\in \mathrm{Aff}_{w, w}.
\end{align*}
We first prove the approximation property of composition of function classes.
\begin{proposition}[Approximation Property of Composition of Function Classes]\label{prop_composition_classes}
	For measurable compact sets $X_i\subset\mathbb{R}^{d_i}$, $\mathcal{F}_i, \mathcal{G}_i\subset C(X_i, \mathbb{R}^{d_{i+1}})$, if $\mathcal{G}_i\underset{\mathrm{sup}}{\prec}\mathcal{F}_i$ holds for any $i\in \{1, 2, \cdots, N\}$, and for any $f_{i}\in\mathcal{F}_i$, $g_{i}\in \mathcal{G}_i$, it holds that $f_{i}(X_i)\subset X_{i+1}$, $g_{i}(X_i)\subset X_{i+1}$. Then we have 
	\begin{align*}
		\mathcal{G}_N\circ \cdots \circ \mathcal{G}_1\underset{\mathrm{sup}}{\prec}\mathcal{F}_N\circ \cdots \circ \mathcal{F}_1
	\end{align*}
\end{proposition}
\begin{proof}
	We can prove this by induction on $N\in \mathbb{N}_+$, and it is clear that it suffices to establish the base case $N=2$. We first prove that $\mathcal{G}_2\circ \mathcal{G}_1\underset{\mathrm{sup}}{\prec} \mathcal{F}_2\circ \mathcal{F}_1$. Since for arbitrary $f_i\in \mathcal{F}_i, g_i\in \mathcal{G}_i$, it holds that
	\begin{align*}
		g_2\circ g_1 - f_2\circ f_1 = (g_2\circ g_1- g_2\circ f_1)+(g_2\circ f_1-f_2\circ f_1).
	\end{align*}
	Let $M=\{y\in X_2\vert\mathrm{dist}(y,g_1(K))\leq 1\}$, then $M$ is compact. Since $g_2$ is continuous, it is uniformly continuous on $M$, then there exists $0<\delta<\frac{1}{2}$ such that when $\vert\vert g_1-f_1 \vert\vert_{\mathrm{sup}, K}<\delta$, $\vert\vert g_2\circ g_1 -g_2\circ f_1\vert\vert_{\mathrm{sup}, K}<\frac{\varepsilon}{2}$. It follows immediately by the assumption $\mathcal{G}_1\underset{\mathrm{sup}}{\prec}\mathcal{F}_1$, there exists $f_1$ such that $\vert\vert g_2\circ g_1- g_2\circ f_1\vert\vert_{\mathrm{sup}, K} <\frac{\varepsilon}{2}$. Furthermore, by the assumption $\mathcal{G}_2\underset{\mathrm{sup}}{\prec}\mathcal{F}_2$, we can choose $f_2\in \mathcal{F}_2$ such that $\vert\vert g_2-f_2\vert\vert_{\mathrm{sup}, f_1(K)}<\frac{\varepsilon}{2}$. Therefore, we can conclude that $\mathcal{G}_2\circ \mathcal{G}_1\underset{\mathrm{sup}}{\prec} \mathcal{F}_2\circ \mathcal{F}_1$, which completes the proof.
\end{proof}

We show that three classes of affine transformations $T_{W, B}\in \mathrm{Aff}_{w, w}$ can be directly compactly approximated by compositions of residual blocks with inner width 1. These cases are stated in the following proposition.
\begin{proposition}[Approximate Affine Transformations]\label{prop_appro_affine}
	Let $\sigma=\{\sigma_\lambda\vert \lambda\in\Lambda\}$ be an activation function family. For any affine transformation $T_{W, B}\in \mathrm{Aff}_{w, w}$, when $W$ satisfies one of the following conditions:
	\begin{itemize}
		\item[(a)] $W$ is a diagonal matrix,
		\item[(b)] $W$ is a lower or upper triangular matrix whose diagonal entries are nonzero,
		\item[(c)] $W$ is a permutation matrix,
	\end{itemize}
	then $T_{W, B}$ can be compactly approximated on compact sets by compositions of residual blocks with inner width 1 both in $\mathrm{sup}$ norm and $L^p$ norm, i.e.,
	\begin{align*}
		T_{W, B}\underset{\mathrm{sup}}{\prec} \mathrm{CRB}_{w, 1}^{\sigma}, \enspace T_{W, B}\underset{p}{\prec} \mathrm{CRB}_{w, 1}^{\sigma}.
	\end{align*}
\end{proposition}
\begin{proof}
	The proof is provided in Appendix \ref{appendix_proof_appro_affine}.
\end{proof}

Proposition \ref{prop_appro_affine} implies the following approximation property for affine transformations.
\begin{theorem}[Approximate Affine Transformations by Compositions of Residual Blocks]\label{thm_appro_affine}
	For any affine transformation $T_{W, B}\in \mathrm{Aff}_{w, w}$, we have
	\begin{align*}
		T_{W, B}\underset{\mathrm{sup}}{\prec} \mathrm{CRB}^{\sigma}_{w, 1}, \enspace T_{W, B}\underset{p}{\prec} \mathrm{CRB}^{\sigma}_{w, 1}.
	\end{align*}
\end{theorem}
\begin{proof}
	Since compact uniform approximation implies $L^p$ approximation on compact sets, it suffices to prove the sup norm case. By the Singular Value Decomposition (SVD) \citep{LinearAlgebra_2019}, every matrix $W$ can be written as $W=U^T\Sigma V$, where $\Sigma$ is a diagonal matrix with non-negative entries, and $U, V$ are orthogonal matrices. The PLU-factorization \citep{LinearAlgebra_2019} states that, for every non-singular $A\in \mathbb{R}^{w\times w}$, there exists a permutation matrix $P$, a lower triangular matrix $L$ and an upper triangular matrix $U$ such that $A = PLU$. Therefore, there exist the corresponding matrices $P_1, P_2, L_1, L_2, U_1, U_2$ such that $U^T = P_1L_1U_1$, $V = P_2L_2U_2$.
	
	By employing Proposition \ref{prop_composition_classes} and Proposition \ref{prop_appro_affine}, we can show that for any $W=U^T\Sigma V$ and $B\in \mathbb{R}^w$, 
	\begin{align*}
		T_{W, B} = T_{I_w, B} \circ T_{P_1, \mathbf{0}_w} \circ T_{L_1, \mathbf{0}_w} \circ T_{U_1, \mathbf{0}_w} \circ T_{\Sigma, \mathbf{0}_w}\circ T_{P_2, \mathbf{0}_w}\circ T_{L_2, \mathbf{0}_w}\circ T_{U_2, \mathbf{0}_w}
	\end{align*}
	can be approximated by a sequence of functions $\{\mathcal{B}_i\}_{i\in \mathbb{N}_+}$, where each $\mathcal{B}_i$ is a composition of residual blocks.
\end{proof}

\subsection{The Ability of Compactly Approximating Nonlinear Maps}
In this subsection, we show our approximation results for certain nonlinear maps, including the maps induced by $\mathrm{LeakyReLU}$, $\mathrm{ABS}$ activation functions.

\begin{definition}[Definition of ReLU-like]\label{def_ReLU-like}
	A set of activation functions $\sigma=\{\sigma_i\vert i\in \Lambda\}$ is called ReLU-like, if for any $\sigma_i\in \sigma$, any $\varepsilon>0$, and any compact interval $K=[k_1, k_2]$, there exist affine transformations $T_{W_1, B_1}, T_{W_2, B_2}\in \mathrm{Aff}_{1, 1}$ such that
	\begin{align*}
		\vert\vert T_{W_2, B_2}\circ \sigma_i \circ T_{W_1, B_1} - \mathrm{ReLU}\vert\vert_{\mathrm{sup}, K}<\varepsilon.
	\end{align*}
\end{definition}
The following proposition shows that a lot of commonly used activation functions are ReLU-like, including ELU, CELU, SELU, Softplus, GELU, SiLU, Mish, ReLU6, Softshrink, HardSigmoid, HardTanh, HardSwish.
\begin{proposition}[ReLU-like Activation Functions]\label{prop_ReLU-like}
	Each of the following activation families is ReLU-like:
	\begin{align*}
		\sigma=\mathrm{ELU}, \mathrm{CELU}, \mathrm{SELU}, \mathrm{Softplus}, \mathrm{GELU}, \mathrm{SiLU}, \mathrm{Mish},\\
		\mathrm{ReLU6}, \mathrm{Softshrink}, \mathrm{HardSigmoid}, \mathrm{HardTanh}, \mathrm{HardSwish}.
	\end{align*} 
\end{proposition}
\begin{proof}
	The proof is provided in the Appendix \ref{appendix_proof_ReLU-like}.
\end{proof}

Before we prove the approximation property of the composition of residual blocks, we first introduce a proposition of $\mathrm{LeakyReLU}$ proved in \citep{Self_ELUNN}.
\begin{proposition}[Approximate LeakyReLU]\label{prop_LRk_to_LR}
	For any fixed $k\in (0, 1)\cup(1, +\infty)$, we have
	\begin{align*}
		\mathrm{LeakyReLU}=\{\mathrm{LeakyReLU}_a\vert a\in (0, 1)\cup(1, +\infty)\}\underset{\mathrm{sup}}{\prec}\mathrm{CRB}^{\mathrm{LeakyReLU}_k}_{w, 1}.
	\end{align*}
\end{proposition}
\begin{proof}
	Fix $k\in (0,1)\cup(1,+\infty)$, we have
	\begin{align*}
		x+\mathrm{LeakyReLU}_k(x)&=
		\left\{
		\begin{aligned}
			2x,& \enspace x> 0,\\
			(1+k)x, & \enspace x\leq0.
		\end{aligned}
		\right.\\
		&=2* \mathrm{LeakyReLU}_{\frac{1+k}{2}}(x).
	\end{align*}
	
	Let $a\in (0, 1)\cup (1, +\infty)$ be arbitrary, for any $\varepsilon>0$ and compact set $K\subset\mathbb{R}$, the proof of Theorem 14 in \citep{Self_ELUNN} implies that there exist affine transformations $\{T_{W_i, B_i}\}_{i\in\{1, 2, \cdots, N\}}$ that
	\begin{align*}
		\vert\vert \mathrm{LeakyReLU}_a - T_{W_N, B_N} \circ \mathrm{LeakyReLU}_{\frac{1+k}{2}} \circ \cdots \\
		\circ \mathrm{LeakyReLU}_{\frac{1+k}{2}} \circ T_{W_1, B_1}\vert\vert_{\mathrm{sup}, K}<\varepsilon.
	\end{align*} 
	
	By employing Theorem \ref{thm_appro_affine}, each affine transformation $T_{W_i, B_i}$ can be approximated by compositions of residual blocks, which completes the proof.
\end{proof}

By employing Proposition \ref{prop_ReLU-like} and Proposition \ref{prop_LRk_to_LR}, we show that the compositions of residual blocks can approximate the piecewise nonlinear functions introduced in the following theorem.
\begin{theorem}[Approximate Piecewise Function]\label{thm_slopeline}
	For any $a\in \mathbb{R}$, define $f_a:\mathbb{R}^w\to \mathbb{R}^w$ by
	$$
	f_a(x) = (f_1(x), f_2(x), \cdots, f_w(x)), \enspace
	f_i(x)=
	\left\{
	\begin{aligned}
		x_i,&\enspace x_i>0,\\
		ax_i,&\enspace x_i\leq 0.
	\end{aligned}
	\right.
	,
	$$
	if $\sigma=\mathrm{LeakyReLU}$, $\mathrm{ReLU}$, or if $\sigma$ is ReLU-like, then we have 
	\begin{align*}
		f_a\underset{\mathrm{sup}}{\prec}\mathrm{CRB}^{\sigma}_{w, 1}
	\end{align*}
\end{theorem}
\begin{proof}
	The proof is provided in Appendix \ref{appendix_proof_slopeline}.
\end{proof}

\section{Minimum Block Width of ResNet}
In this section, we prove that continuous functions can be approximated by the ResNets on compact sets and determine the required block width $w(n, m)$ for approximation in the $L^p$ norm and the $\mathrm{sup}$ norm, especially when the inner width $\bar{w}$ is fixed as 1. Our main concern is the minimum block width of ResNets with inner width 1.

We first give the formal definition of the minimum block width of ResNets.
\begin{definition}[Minimum Block Width]\label{def_minimumwidth}
	For $1\leq p<+\infty$, the indicators $w_{\min}^{\mathrm{sup}}(n, m, \sigma)$ and $w_{\min}^{p}(n, m, \sigma)\enspace(1\leq p <+\infty)$ are defined by:
	\begin{align*}
		w_{\min}^{\mathrm{sup}}(n, m, \sigma)&:=\min\{l\in \mathbb{N}_+\vert C(\mathbb{R}^n, \mathbb{R}^m)\underset{\mathrm{sup}}{\prec}\mathrm{Res}_{n, m, l, 1}^{\sigma}\}\\
		w_{\min}^{p}(n, m, \sigma)&:=\min\{l\in \mathbb{N}_+\vert L^p_{\mathrm{loc}}(\mathbb{R}^n, \mathbb{R}^m)\underset{p}{\prec}\mathrm{Res}_{n, m, l, 1}^{\sigma}\}
	\end{align*}
\end{definition}

Before discussing the minimum block width, we recall the definition of topological embedding \citep{Self_ELUNN}.
\begin{definition}[Embedding]\label{def_embedding}
	An injective continuous map $f:X\to Y$ is called an embedding if the induced map $f:X\to f(X)$ is a homeomorphism, where $f(X)$ is endowed with the subspace topology inherited from $Y$. When $X$ and $Y$ are smooth manifolds, we denote $\mathrm{Emb}(X,Y)$ as the set of smooth embeddings from $X$ to $Y$.
\end{definition}

We also recall the standard definition of MLP \citep{Self_ELUNN}. We denote $\mathrm{MLP}^{\sigma}_{n, m, k}$ as the set of all such MLPs with input set $\mathbb{R}^n$, output set $\mathbb{R}^m$, arbitrary depth, and all hidden layer widths at most $k$, where all the activation functions are chosen from the set $\sigma$.

\subsection{Minimum Block Width of ResNet for $L^p$ Approximation}
In this subsection, for $\mathrm{LeakyReLU}, \mathrm{ReLU}$ and ReLU-like activation functions, we establish the lower and upper bounds of the minimum block width for $L^p$ approximation, for all input and output dimensions $d_x$ and $d_y$.

In the following theorem, we first show the upper bound of the minimum block width in the $L^p$ sense.
\begin{theorem}[Upper Bound of Minimum Block Width for $L^p$ UAP]\label{thm_upper_Lp}
	For $p\in [1, +\infty)$, if $\sigma=\mathrm{LeakyReLU}$, $\mathrm{ReLU}$ or $\sigma$ is ReLU-like, then we have
	\begin{align*}
		w_{\min}^{p}(d_x, d_y, \sigma)\leq \max\{d_x, d_y\}.
	\end{align*}
\end{theorem}
\begin{proof}
	The discussion in Section 5.2 of \citep{Cai} implies that, for any compact set $K\subset \mathbb{R}^{d_x}$, 
	\begin{align*}
		L^p(K, \mathbb{R}^{d_y})\underset{p}{\prec}\mathrm{MLP}^{\mathrm{LeakyReLU} \cup \{\mathrm{ABS}\}}_{d_x, d_y, \max\{d_x, d_y\}}.
	\end{align*}
	Thus, for any $f\in L^p(K, \mathbb{R}^{d_y})$ and any $\varepsilon>0$, there exist $T_{W_N, B_N}\in \mathrm{Aff}_{\max\{d_x, d_y\}, d_y}$, $T_{W_0, B_0}\in \mathrm{Aff}_{d_x, \max\{d_x, d_y\}}$, $T_{W_i, B_i}\in \mathrm{Aff}_{\max\{d_x, d_y\}, \max\{d_x, d_y\}}$ and $\sigma_j\in \mathrm{LeakyReLU}\cup \{\mathrm{ABS}\}$ such that 
	\begin{align*}
		\vert\vert T_{W_N, B_N} \circ \sigma_N \circ \cdots \circ T_{W_1, B_1}\circ \sigma_1 \circ T_{W_0, B_0} - f \vert\vert_{p, K}<\frac{\varepsilon}{2}.
	\end{align*}
	By Theorem \ref{thm_appro_affine} and Theorem \ref{thm_slopeline}, all $T_{W_i, B_i}$ and $f_a$ in Theorem \ref{thm_slopeline} can be compactly approximated on compact sets by the compositions of residual blocks. Indeed, for any $k\in (0, 1)\cup (1,+\infty)$, $\mathrm{LeakyReLU}_k$ and $\mathrm{ABS}$ correspond to the $f_a$ in Theorem \ref{thm_slopeline} for $a=k$ and $a=-1$, respectively. Therefore, for any compact set $K\subset \mathbb{R}^{d_x}$ and $\varepsilon>0$, there exists a $\phi\in \mathrm{CRB}^{\sigma}_{\max\{d_x, d_y\}, 1}$ such that
	\begin{align*}
		\vert\vert T_{W_N, B_N}\circ \phi \circ T_{W_0, B_0}-T_{W_N, B_N} \circ \sigma_N \circ \cdots \circ T_{W_1, B_1}\circ \sigma_1 \circ T_{W_0, B_0}\vert\vert_{p, K}<\frac{\varepsilon}{2}.
	\end{align*}
	By Definition \ref{def_ResNet}, the map $T_{W_N, B_N}\circ\phi\circ T_{W_0, B_0}$ belongs to $\mathrm{Res}_{d_x,d_y,\max\{d_x,d_y\},1}^{\sigma}$. Thus, $w_{\min}^{p}(d_x, d_y, \sigma)\leq \max\{d_x, d_y\}$.
\end{proof}

Motivated by the discussion in Lemma 1 of \citep{Cai}, we prove lower bounds for the minimum block width of ResNets both in the $L^p$ norm and the $\mathrm{sup}$ norm sense.
\begin{theorem}[Lower Bound of Minimum Block Width]\label{thm_lower_C_Lp}
	For arbitrary activation function family $\sigma=\{\sigma_i\vert i\in \Lambda\}$ and $1\leq p<+\infty$,
	\begin{align*}
		w_{\min}^{\mathrm{sup}}(d_x, d_y, \sigma)\geq \max\{d_x, d_y\}, \enspace w_{\min}^{p}(d_x, d_y, \sigma) \geq \max\{d_x, d_y\}
	\end{align*}
\end{theorem}
\begin{proof}	
	Since $\mathcal{G}\underset{\mathrm{sup}}{\prec}\mathcal{F}$ implies the $L^p$ approximation on any compact set, it suffices to prove the $L^p$ approximation case. 
	
	To prove that $L^p(K, \mathbb{R}^{d_y})\underset{p}{\nprec}\mathrm{Res}_{d_x, d_y, w, 1}$ whenever $w<\max\{d_x, d_y\}$, it suffices to exhibit a compact set $K\subset \mathbb{R}^{d_x}$ and a function
	$f^*\in L^p(K,\mathbb R^{d_y})$ such that no ResNet with block width $w$ can approximate
	$f^*$ in $L^p(K)$. We first suppose that $w\leq d_x-1$. It suffices to prove the case that $d_y=1$, and we let $f^*(x)=\vert\vert x \vert\vert^2$ and choose the compact set $K=[-2, 2]^{d_x}$. By the definition \ref{def_ResNet}, for ResNet $f\in \mathrm{Res}_{d_x, d_y, w, 1}$ denoted by
	\begin{align*}
		f=T_{W_b, B_b} \circ \phi \circ T_{W_a, B_a},
	\end{align*} 
	since $W_a\in \mathbb{R}^{w, d_x}, B_a\in \mathbb{R}^w$ and the block width $w<d_x$, there exists a $v\in \mathbb{R}^{d_x}$ such that 
	\begin{align*}
		W_a v=0, \enspace \vert\vert v\vert\vert=1.
	\end{align*} 
	It follows that $f(x)=f(x+\lambda v)$ holds for all $x\in K$ and $\lambda\in \mathbb{R}$. For the sets
	\begin{align*}
		A_1=\{x\vert\enspace \vert\vert x\vert\vert\leq 0.1\}, \enspace A_2=\{x\vert \enspace\vert\vert x-v\vert\vert\leq 0.1\},
	\end{align*}
	let $\mu$ be the Lebesgue measure on $\mathbb{R}^{d_x}$. By the monotonicity of $L
	^p$-norm on finite measure spaces \citep{Williams}, we have
	\begin{align*}
		\left(\frac{1}{\mu(K)} \int_{K} \vert\vert f(x)-f^*(x) \vert\vert^p d\mu \right)^{\frac{1}{p}}\geq \left(\frac{1}{\mu(K)} \int_{K} \vert\vert f(x)-f^*(x) \vert\vert d\mu \right) \enspace (1\leq p<+\infty).
	\end{align*}
	Therefore, it follows that 
	\begin{align*}
		\vert\vert f-f^*\vert\vert_{p, K}&\geq (\mu(K))^{\frac{1-p}{p}}\int_{K}\vert\vert f(x)-f^*(x)\vert\vert d\mu\\
		&\geq (\mu(K))^{\frac{1-p}{p}}\left(\int_{A_1}\vert\vert f(x)-f^*(x)\vert\vert d\mu+ \int_{A_2}\vert\vert f(x)-f^*(x)\vert\vert d\mu\right)\\
		&\geq (\mu(K))^{\frac{1-p}{p}} \int_{A_1}\left( \vert\vert f(x)-f^*(x)\vert\vert + \vert\vert f(x+v)-f^*(x+v)\vert\vert \right)d\mu\\
		&\geq (\mu(K))^{\frac{1-p}{p}}\int_{A_1} \vert\vert f^*(x)-f^*(x+v) \vert\vert d\mu \\
		&\geq \frac{4}{5}\mu(A_1)(\mu(K))^{\frac{1-p}{p}}.
	\end{align*}
	Thus, no ResNet with block width $w<d_x$ can approximate $f^*$ in the $L^p$ norm on $K$, and therefore $w_{\min}^{p}(d_x,d_y,\sigma)\ge d_x$.
	
	To prove that $L^p(K, \mathbb{R}^{d_y})\underset{p}{\nprec}\mathrm{Res}^\sigma_{d_x, d_y, w, 1}$ for the case that $w<d_y$, we assume that $d_x\leq w < d_y$, the compact set $K=[0, 1]^{d_x}$, then we define the map $g:\mathbb{R}^{d_x}\to \mathbb{R}^{d_y}$ and $\gamma:[0, 1]\to \mathbb{R}^{d_y}$ by
	\begin{align*}
		g(x)=(x_1, x_1^2, \cdots, x_1^{d_y}), \enspace \gamma(t)=(t, t^2, \cdots, t^{d_y}).
	\end{align*} 
	For any $f=T_{W_b,B_b}\circ\phi\circ T_{W_a,B_a}\in\mathrm{Res}_{d_x,d_y,w,1}^{\sigma}$, the image $f(K)$ is contained in an affine subspace of dimension at most $w$. Since $w<d_y$, it is contained in a hyperplane $S$ of the form
	\begin{align*}
		S=\{z\in \mathbb{R}^{d_y}\vert \langle z, a\rangle=b, a=(a_1, \cdots, a_{d_y})\}
	\end{align*}
	where $\vert\vert a \vert\vert =1$ and $b\in \mathbb{R}$.
	
	By employing the Fubini Theorem \citep{Axler2020MIRA}, we have 
	\begin{align*}
		\vert\vert f-g\vert\vert_{p, K}^p&\geq \int_{K}(\mathrm{dist}(g(x), f(K)))^p dx\\
		&=\int_{[0, 1]}\int_{[0, 1]^{d_x-1}} (\mathrm{dist}(\gamma(t), f(K)))^p dzdt\\
		&=\int_{[0, 1]} (\mathrm{dist}(\gamma(t), f(K)))^pdt.
	\end{align*}
	Since for any $t\in [0, 1]$, we have $\mathrm{dist}(\gamma(t), f(K))\geq\mathrm{dist}(\gamma(t), S)=\vert\langle \gamma(t), a\rangle-b\vert$. Therefore, by the properties of finite dimensional polynomial, we can conclude that
	\begin{align*}
		\int_{[0, 1]} (\mathrm{dist}(\gamma(t), f(K)))^p dt&\geq \int_{[0, 1]}\vert \langle a, \gamma(t)\rangle -b\vert^p dt\\
		&= \int_{[0, 1]}\vert a_1t+\cdots+a_{d_y}t^{d_y}-b\vert^p dt\\
		&\geq \eta_{d_y, p} \vert\vert c\vert\vert^p\\
		&\geq \eta_{d_y, p},
	\end{align*}
	where $c=(-b, a_1, \cdots, a_{d_y})\in \mathbb{R}^{d_y+1}$, and 
	\begin{align*}
		\eta_{d_y, p}:=\min_{\vert\vert c \vert\vert=1}\frac{\int_{[0, 1]}\vert a_1t+\cdots+a_{d_y}t^{d_y}-b\vert^p dt}{\vert\vert c \vert\vert^p}>0
	\end{align*}
	is a constant. Consequently, $\vert\vert f-g\vert\vert_{p, K}\geq(\eta_{d_y, p})^{\frac{1}{p}}>0$, which proves that no such block width $w<d_y$ can achieve the $L^p$ UAP.
\end{proof}
In fact the proof implies that for arbitrary inner width $\bar{w}>1$, the ResNets with block width less than $\max\{d_x, d_y\}$ cannot approximate all functions in $L^p(K, \mathbb{R}^{d_y})$ or  $C(K, \mathbb{R}^{d_y})$.

By employing Theorem \ref{thm_upper_Lp} and Theorem \ref{thm_lower_C_Lp}, we can derive the exact minimum block width for the $L^p$ approximation.
\begin{theorem}[Minimum Block Width for $L_p$ UAP]\label{thm_minwidth_Lp}
	For activation functions $\sigma=\mathrm{LeakyReLU}, \mathrm{ReLU}$ or ReLU-like $\sigma$, we have 
	\begin{align*}
		w_{\min}^{p}(d_x, d_y, \sigma)=\max\{d_x, d_y\}
	\end{align*}
\end{theorem}
\begin{proof}
	The result follows immediately from Theorem \ref{thm_upper_Lp} and Theorem \ref{thm_lower_C_Lp}.
\end{proof}

\subsection{Minimum Block Width of ResNet for Uniform Approximation}
In this subsection, we establish the results of lower and upper bounds for the uniform approximation.

Motivated by the approach used in \citep{Hwang_C0} to study the minimum width of MLPs, a natural strategy is to first embed the input set from $\mathbb{R}^{d_x}$ in to a higher dimensional space $\mathbb{R}^w$, and then investigate the approximation ability of compositions of residual blocks on $\mathbb{R}^w$ for diffeomorphisms. It allows us to prove the approximation ability of ResNets for all target functions. For this purpose, we recall the notion of minimum embedding width introduced in \citep{Hwang_C0}, and then prove the compact uniform approximation property of ResNets.
\begin{definition}[Minimum Embedding Width]\label{def_min-emb-width}
	For $n,l,m\in\mathbb N_+$, define
	\begin{align*}
		A_{n, l, m}=\{g\vert g= \pi_{l, m}\circ f, \enspace f\in\mathrm{Emb}([0, 1]^n, \mathbb{R}^{l})\}
	\end{align*}
	The minimum embedding width is then defined by $\Omega(n, m)$:
	\begin{align*}
		\Omega(n, m):=\min \{l\in \mathbb{N}_+\vert C([0, 1]^n, \mathbb{R}^m)\underset{\mathrm{sup}}{\prec} A_{n, l, m}\}.
	\end{align*}
\end{definition}

Before discussing the lower and upper bounds of $\Omega(n, m)$, we recall that continuous maps can be uniformly approximated by smooth maps on compact sets.
\begin{proposition}[Mollification of Continuous Function]\label{prop_smooth}
	For any $f\in C(\mathbb{R}^n, \mathbb{R}^m)$, compact set $K\subset \mathbb{R}^n$ and $\varepsilon>0$, we can find a $f^\varepsilon\in C^{\infty}(\mathbb{R}^n, \mathbb{R}^m)$ such that
	\begin{align*}
		\vert\vert f^\varepsilon - f\vert\vert_{\mathrm{sup}, K}<\varepsilon.
	\end{align*}
\end{proposition}
\begin{proof}
	For any $f=(f_1, f_2, \cdots, f_m)\in C(\mathbb{R}^n,\mathbb{R}^m)$, any compact set $K\subset\mathbb{R}^n$ and any $\varepsilon>0$, by the properties of mollifiers (see \citep{Evans}, Appendix C.5, Theorem 7), there exists a $f_i^\varepsilon\in C^{\infty}(\mathbb{R}^n, \mathbb{R})$ such that $\vert\vert f_i^\varepsilon-f_i\vert\vert_{\mathrm{sup}, K}<\frac{\varepsilon}{\sqrt{m}}$, then it is easy to see that, for $f^\varepsilon=(f_1^\varepsilon, f_2^\varepsilon, \cdots, f_m^{\varepsilon})$, we have $\vert\vert f^\varepsilon - f\vert\vert_{\mathrm{sup}, K}<\varepsilon.$
\end{proof}

By employing Proposition \ref{prop_smooth}, we can establish the upper bound of the $\Omega(n, m)$.
\begin{proposition}[Upper Bound of $\Omega(n, m)$]\label{prop_upper_Omega}
	For any $n, m\in \mathbb{N}_+$, we have
	\begin{align*}
		\Omega(n, m)\leq \min\{n+m, \max\{2n+1, m\}\}.
	\end{align*}
\end{proposition}
\begin{proof}
	For all $n, m\in \mathbb{N}_+$, a compact set $K=[0, 1]^n$, any $\varepsilon>0$, and any continuous $f: K\to \mathbb{R}^m=(f_1(x), f_2(x), \cdots, f_m(x))$, by applying the Tietze extension theorem \citep{Munkres}, $f$ admits a continuous extension $\widetilde f\in C(\mathbb R^n,\mathbb R^m)$. By employing Proposition \ref{prop_smooth}, we can find $f^\varepsilon=(f_1^\varepsilon, f_2^\varepsilon, \cdots, f_m^\varepsilon)\in C^{\infty}(\mathbb{R}^n, \mathbb{R}^m)$ such that $\vert\vert f^\varepsilon\vert_{K}-f \vert\vert_{\mathrm{sup}, K} = \vert\vert f^\varepsilon-\widetilde f \vert\vert_{\mathrm{sup}, K}<\varepsilon$.
	
	Define a smooth embedding $g:K\to \mathbb{R}^{n+m}$ by:
	\begin{align*}
		g(x_1, x_2, \cdots, x_n)=(f_1^\varepsilon(x), \cdots, f_m^\varepsilon(x), x_1, x_2, \cdots, x_n).
	\end{align*}
	Then it is easy to see that $g\in \mathrm{Emb}(K, \mathbb{R}^{n+m})$, and $\pi_{n+m, m}\circ g$ satisfies that 
	\begin{align*}
		\vert\vert \pi_{n+m, m}\circ g -f \vert\vert_{\mathrm{sup}, K}<\varepsilon,
	\end{align*}
	which implies that $\Omega(n, m)\leq n+m$.
	
	When $m>2n$, by Lemma 4.7 of \citep{Hwang_C0}, for any compact set $K\subset \mathbb{R}^n$ and $\varepsilon>0$, there exists a smooth embedding $g:K\to \mathbb{R}^{m}$ such that $\vert\vert f-g\vert\vert_{\sup, K}<\varepsilon$, which implies $\Omega(n, m)\leq m$. 
	
	When $m \leq 2n$, we can extend the $f(x)=(f_1(x), f_2(x), \cdots, f_m(x))$ to the $f^\prime:K\to \mathbb{R}^{2n+1}$:
	\begin{align*}
		f^\prime(x)=(f_1(x), f_2(x), \cdots, f_m(x), 0, \cdots, 0).
	\end{align*}
	Then there exists a smooth embedding $g^\prime:K\to \mathbb{R}^{2n+1}$ such that $\vert\vert f^\prime-g^\prime\vert\vert_{\sup, K}<\varepsilon.$
	Then $g=\pi_{2n+1, m}\circ g^\prime$ is a smooth map from $K$ to $\mathbb{R}^m$, and it holds that $\vert\vert f-g\vert\vert_{\sup, K}<\varepsilon,$
	which gives the bound $\Omega(n, m)\leq 2n+1$. Therefore, it follows that 
	\begin{align*}
		\Omega(n, m)\leq \min\{n+m, \max\{2n+1, m\}\}.
	\end{align*}
\end{proof}

Before we establish the upper bound of minimum block width, we first prove that $C^2$-diffeomorphisms can be approximated by compositions of residual blocks.
\begin{proposition}[Approximate Diffeomorphisms by Compositions of Residual Blocks]\label{prop_D^2_appro}
	Let $w\geq 2$, for activation functions $\sigma=\mathrm{LeakyReLU}, \mathrm{ReLU}$ or ReLU-like activations $\sigma$, we have
	\begin{align*}
		\mathcal{D}^2(\mathbb{R}^w)\underset{\mathrm{sup}}{\prec}\mathrm{CRB}_{w, 1}^{\sigma}
	\end{align*}
\end{proposition}
\begin{proof}
	By Theorem 1 of \citep{Coupling_Flow_NeuralPS}, we can know that
	\begin{align*}
		\mathcal{D}^2(\mathbb{R}^w)\underset{\sup}{\prec} \mathrm{INN}_{\mathcal{S}_c^\infty(\mathbb{R}^w)}.
	\end{align*}
	Lemma B.2 in \citep{Hwang_C0} implies that 
	\begin{align*}
		\mathcal{S}_{c}^\infty(\mathbb{R}^w)\underset{\mathrm{sup}}{\prec}\mathrm{MLP}^{\mathrm{LeakyReLU}}_{w, w, w}.
	\end{align*}
	Consequently, for any $F\in \mathcal{D}^2(\mathbb{R}^w)$, any compact set $K\subset\mathbb{R}^w$ and any $\varepsilon>0$, there exist $T_{W_i, B_i}\in \mathrm{Aff}_{w, w} \enspace(i=0, 1, 2, \cdots, N)$ and $\mathrm{LeakyReLU}_j \enspace(j=1, 2, \cdots, N)$ such that 
	\begin{align*}
		\vert\vert T_{W_N, B_N} \circ \mathrm{LeakyReLU}_{a_N}\circ \cdots \circ T_{W_1, B_1} \circ \mathrm{LeakyReLU}_{a_1} \circ T_{W_0, B_0} - F \vert\vert_{\mathrm{sup}, K}<\frac{\varepsilon}{2}.
	\end{align*}
	For any $a\in (0, 1)\cup (1, +\infty)$, the $\mathrm{LeakyReLU}_a$ corresponds to the $f_a$ in Theorem \ref{thm_slopeline}. By Theorem \ref{thm_appro_affine} and Theorem \ref{thm_slopeline}, for any compact set $K\subset \mathbb{R}^{w}$ and $\varepsilon>0$, when activation functions $\sigma=\mathrm{LeakyReLU}, \mathrm{ReLU}$ or $\sigma$ is ReLU-like, we can show that each $T_{W_i, B_i}$ and $\mathrm{LeakyReLU}_{a_j} \enspace (i, j\in\{1, 2, \cdots, N\})$ can be approximated by the compositions of residual blocks. 
	
	Therefore, we can show that any $C^2$-diffeomorphism can be approximated by the set of compositions of residual blocks, i.e., $\mathcal{D}^2(\mathbb{R}^w)\underset{\mathrm{sup}}{\prec}\mathrm{CRB}_{w, 1}^{\sigma}$
\end{proof}

By Proposition \ref{prop_D^2_appro}, we now derive the upper bound of the block width of the ResNet for the universal approximation property on $C([0, 1]^n, \mathbb{R}^m)$ in the $\mathrm{sup}$ norm sense.
\begin{theorem}[Upper Bound of Minimum Block Width]\label{thm_upper_C}
	For activation functions $\sigma=\mathrm{LeakyReLU}, \mathrm{ReLU}$ or ReLU-like $\sigma$, we have 
	\begin{align*}
		w_{\min}^{\mathrm{sup}}(d_x, d_y, \sigma)\leq \min\{ d_x+d_y, \max\{2d_x+1, d_y\}\}.
	\end{align*}
\end{theorem}
\begin{proof}
	For any compact set $K\subset\mathbb{R}^{d_x}$, we can find a cube $K_1=[a, b]^{d_x}$ such that $K\subset K_1$. By employing the Tietze Extension Theorem \citep{Munkres}, we can extend any continuous map to $K_1$ and then rescale the set $K_1$ to $[0, 1]^{d_x}$, therefore it suffices to consider the case that $K=[0, 1]^{d_x}$.
	
	Theorem C of \citep{Palais} states that for $n\leq m$ and a smooth embedding $f:[0, 1]^n\to \mathbb{R}^m$, then there exists a smooth diffeomorphism $F:\mathbb{R}^m\to \mathbb{R}^m$ such that $F\circ \pi_{n, m} =f.$
	Therefore, for any $f\in C(K, \mathbb{R}^{d_y})$ and any $\varepsilon>0$, by Proposition \ref{prop_upper_Omega} and Theorem C of \citep{Palais}, for $\Omega(d_x, d_y)\leq \min\{ d_x+d_y, \max\{2d_x+1, d_y\}\}$, we can find a smooth diffeomorphism $F:\mathbb{R}^{\Omega(d_x, d_y)} \to \mathbb{R}^{\Omega(d_x, d_y)}$ such that
	\begin{align*}
		\vert\vert \pi_{{\Omega(d_x, d_y)}, d_y} \circ F \circ \pi_{d_x, {\Omega(d_x, d_y)}}-f\vert\vert_{\mathrm{sup}, K}<\frac{\varepsilon}{2}.
	\end{align*}
	It is evident that $\Omega(d_x, d_y)\geq 2$, and by Proposition \ref{prop_D^2_appro}, we can find a composition of residual blocks with block width $w={\Omega(d_x, d_y)}$ such that $F\underset{\mathrm{sup}}{\prec}\mathrm{CRB}_{{\Omega(d_x, d_y)}, 1}^{\sigma}.$
	Then we can find $\phi\in \mathrm{CRB}_{{\Omega(d_x, d_y)}, 1}^{\sigma}$ satisfies that
	\begin{align*}
		\vert\vert \pi_{{\Omega(d_x, d_y)}, d_y} \circ F \circ \pi_{d_x, {\Omega(d_x, d_y)}}-\pi_{{\Omega(d_x, d_y)}, d_y} \circ \phi \circ \pi_{d_x, {\Omega(d_x, d_y)}}\vert\vert_{\mathrm{sup}, K}<\frac{\varepsilon}{2}.
	\end{align*}
	The coordinate maps $\pi_{d_x,\Omega}$ and $\pi_{\Omega,d_y}$ are affine transformations. Indeed, they are represented by
	\begin{align*}
		&W_a=\mathrm{Diag}_{{\Omega(d_x, d_y)}, d_x}, \enspace B_a=\mathbf{0}_{{\Omega(d_x, d_y)}\times 1},\\
		&W_b=\mathrm{Diag}_{d_y, {\Omega(d_x, d_y)}}, \enspace B_b=\mathbf{0}_{d_y\times 1},\\
		&T_{W_a, B_a}=\pi_{d_x, {\Omega(d_x, d_y)}}, \enspace T_{W_b, B_b}=\pi_{{\Omega(d_x, d_y)}, d_y}.
	\end{align*}
	Thus ResNets in $\mathrm{Res}_{d_x, d_y, \Omega(d_x, d_y), 1}^{\sigma}$ can approximate $f$ uniformly on $K$, which completes the proof.
\end{proof}

By employing Theorem \ref{thm_lower_C_Lp} and Theorem \ref{thm_upper_C}, we can derive the lower and upper bounds of the minimum block width for the uniform approximation.
\begin{theorem}[Bounds on the Minimum Block Width for Uniform UAP]\label{thm_minwidth_C}
	For activation functions $\sigma=\mathrm{LeakyReLU}, \mathrm{ReLU}$ or ReLU-like $\sigma$, we have 
	\begin{align*}
		\max\{d_x, d_y\}\leq w_{\min}^{\mathrm{sup}}(d_x, d_y, \sigma)\leq \min\{d_x+d_y, \max\{2d_x+1, d_y\}\}.
	\end{align*}
Especially, for activation functions $\sigma=\mathrm{LeakyReLU}, \mathrm{ReLU}$ or ReLU-like $\sigma$, when $d_y\geq 2d_x+1$, the exact minimum block width for uniform UAP is
\begin{align*}
	w_{\min}^{\mathrm{sup}}(d_x, d_y, \sigma)=d_y.
\end{align*}
\end{theorem}
\begin{proof}
	The first result $\max\{d_x, d_y\}\leq w_{\min}^{\mathrm{sup}}(d_x, d_y, \sigma)\leq \min\{d_x+d_y, \max\{2d_x+1, d_y\}\}$ follows immediately from Theorem \ref{thm_lower_C_Lp} and Theorem \ref{thm_upper_C}. For the second result, for activation functions $\sigma=\mathrm{LeakyReLU}, \mathrm{ReLU}$ or ReLU-like $\sigma$, since when $d_y\geq 2d_x+1$, $\max\{d_x, d_y\}=\min\{d_x+d_y, \max\{2d_x+1, d_y\}\}=d_y$, therefore, we can conclude the exact minimum block width.
\end{proof}

\subsection{Difference between ResNets and MLPs for Uniform Approximation}
In this subsection, we investigate the difference between ResNets and MLPs while approximating some given maps, which indicates the benefit of ResNets for uniform approximation tasks.

For the parameter matrices of the MLPs, by the density of full rank matrices in the matrix space $\mathbb{R}^{m\times n}$, it is shown that when the activation functions are injective or can be approximated by a sequence of injective functions (e.g., LeakyReLU, ReLU), the width-$2n$ MLPs with input-output dimensions $n$ and $2n$ cannot approximate some functions whose image has self-intersections \citep{Self_ELUNN}. For example, for the following function $f:\mathbb{R}\to \mathbb{R}^2$:
\begin{align}\label{eq_self-intersection}
	f(x)=
	\left\{
	\begin{aligned}
		& \left(10x-1, 0\right)
		& 0\leq x\leq \frac{3}{10},\\
		& \left(2, 10x-3\right)
		& \frac{3}{10}\leq x\leq \frac{1}{2},\\
		&\left(7-10x, 2\right)& \frac{1}{2}\leq x \leq \frac{7}{10},\\
		& \left(0, 9-10x\right)& \frac{7}{10}\leq x\leq 1.
	\end{aligned}
	\right.	
	,
\end{align}
\citet{Self_ELUNN} proved that width-2 MLPs with input-output dimensions 1 and 2 cannot approximate this function, regardless of the depth. However, we can construct a ReLU ResNet $R:[0, 1]\to \mathbb{R}^2$ with block width 2 and inner width 1 that $R(x)=f(x)$ holds for any $x\in [0, 1]$. The ReLU ResNet $R:[0, 1]\to \mathbb{R}^2$ is given by
\begin{align*}
   R=T_{W_\beta, B_\beta}\circ \mathcal{B}_{3}\circ \mathcal{B}_{2} \circ \mathcal{B}_{1} \circ T_{W_\alpha, B_\alpha}, \enspace \mathcal{B}_i=\mathrm{Id}_2 + T_{W_{ib}, B_{ib}}\circ \mathrm{ReLU}\circ T_{W_{ia}, B_{ia}},
\end{align*}
and all the parameters are presented in Table \ref{table_appendix_polyline_input_output_parameters} and Table \ref{table_appendix_polyline_block_parameters} (Appendix \ref{appendix_ResNet_parameter}). For the input set $K=[0, 1]$, Figure \ref{fig_Polyline} illustrates how the successive residual blocks transform the embedded set $A_0=T_{W_\alpha, B_\alpha}(K)$ into the target image $f(K)$.
\begin{figure}[htbp]
	\centering
	\includegraphics[height=0.6\linewidth, width=0.6\linewidth]{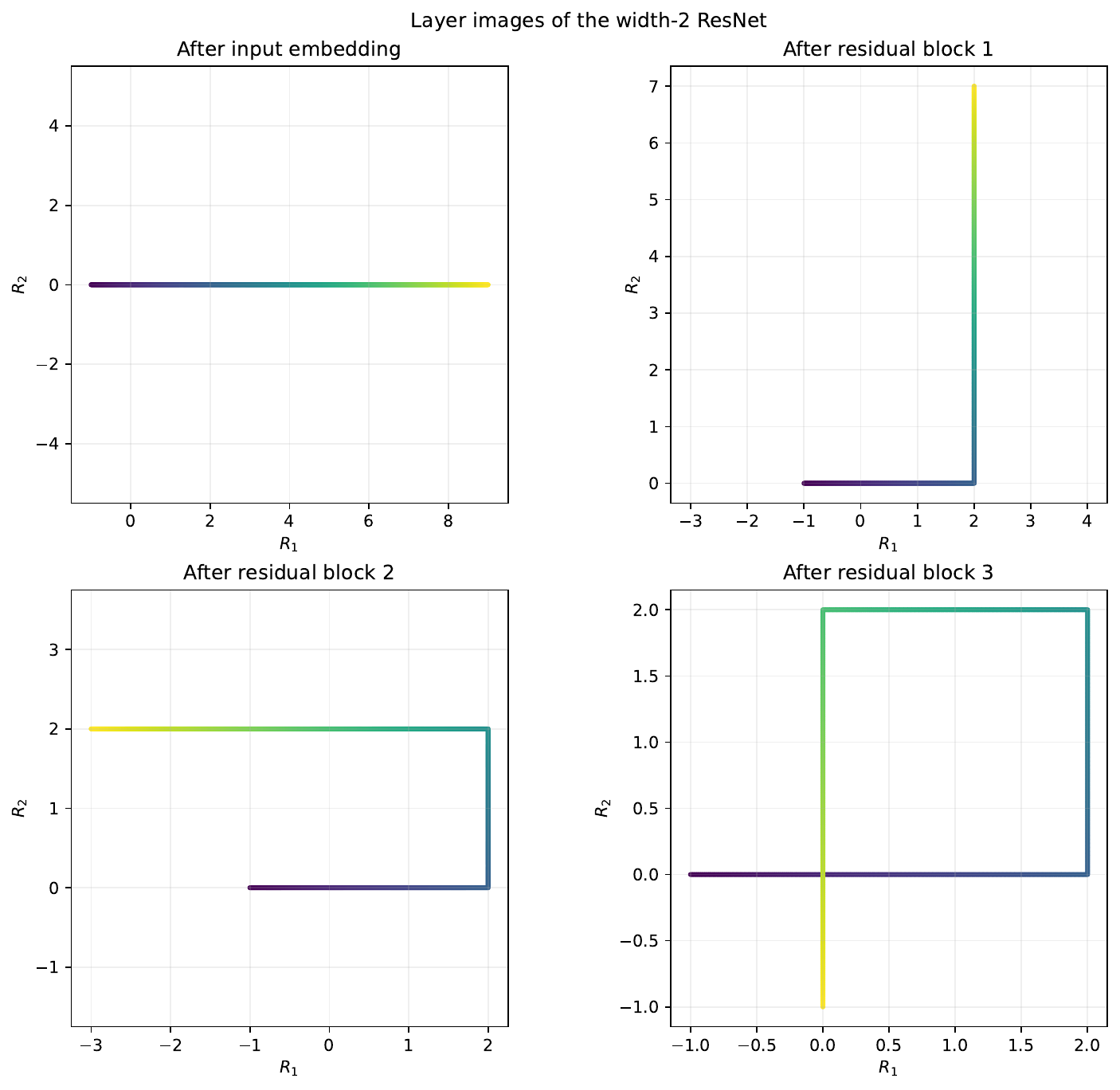}
	\caption{
		Layer-wise evolution of the embedded interval $K=[0,1]$ through the
		width-two ReLU ResNet with inner width one constructed for the target
		function in Equation~(7). The panels, ordered row-wise, show
		$A_0=T_{W_\alpha,B_\alpha}(K)$ and the successive residual-block images
		$A_i=\mathcal{B}_i\circ\cdots\circ\mathcal{B}_1(A_0)$ for
		$i=1,2,3$. The color of each point indicates its original input value
		$x\in[0,1]$ and is kept consistent across all panels. The residual blocks
		progressively fold the embedded interval, and the final image $A_3$
		coincides with the self-intersecting target curve $f(K)$.
	}
	\label{fig_Polyline}
\end{figure}

The main reason for this phenomenon lies in the different roles played by the intermediate layers in these two types of networks. For an activation function that is injective or can be approximated by a sequence of injective functions, if the width of an MLP $w$ and its input and output dimensions $n$ and $m$ satisfy $n\leq w\leq m$, the full-rank perturbation argument in \citep{Self_ELUNN} implies that the MLP with this dimensional setting is an embedding or can be approximated by a sequence of embeddings. Thus, MLPs cannot approximate lots of functions whose images exhibit self-intersections in this case. By contrast, residual blocks are not subject to this restriction and therefore avoid the corresponding topological obstruction. This essential difference indicates the approximation benefits of narrow ResNets when compared with narrow MLPs, which motivates the research for UAP of narrow ResNets.

Besides, the lower bound for the minimum block width indicates a fundamental distinction between ResNets and MLPs. In fact, when both the input and output dimensions are two, \citet{Hwang_C0} proved that for the following continuous function $f:\mathbb{R}^2\to \mathbb{R}^2$:
\begin{align}\label{eq_Hwang}
	f(x_1, x_2)=
	\left\{
	\begin{aligned}
		&\begin{pmatrix}
			1 & -1\\
			0 & 2
		\end{pmatrix}
		\begin{pmatrix}
			x_1\\x_2
		\end{pmatrix} 
		& 0\leq x_2\leq x_1,\\
		&\begin{pmatrix}
			1 & -1\\
			2 & 0
		\end{pmatrix}
		\begin{pmatrix}
			x_1\\x_2
		\end{pmatrix}  & 0\leq x_1\leq x_2,\\
		&-f(x_2, -x_1)& x_1\leq 0 \leq x_2,\\
		&f(-x_1, -x_2)& x_2\leq 0.
	\end{aligned}
	\right.
	,
\end{align}
there exists $\varepsilon_0>0$ that for any $C^2$-diffeomorphism $F:\mathbb{R}^3\to \mathbb{R}^3$, $\vert\vert\pi_{3, 2}\circ F \circ\pi_{2, 3}-f\vert\vert_{\mathrm{sup}, [-2, 2]^2}\geq \varepsilon_0$. This result shows that MLPs of width three cannot uniformly approximate this continuous function (Equation \eqref{eq_Hwang}) on $K=[-2, 2]^2$, hence, they do not possess the uniform UAP over all compact subsets of $\mathbb{R}^2$. 

However, the composition of residual blocks $G=\mathcal{B}_N\circ\cdots\circ \mathcal{B}_1$ is not necessarily a diffeomorphism. Hence a ResNet of block width three does not have to belong to the class $\left\{\pi_{3,2}\circ F\circ \pi_{2,3}\vert F\in \mathcal{D}^2(\mathbb{R}^3)\right\}$. Consequently, the obstruction in \citep{Hwang_C0} for width-three MLPs cannot be directly transferred to ResNets. This reveals the property that lower bounds for MLPs and ResNets are essentially different: the obstruction for narrow MLPs relies on a global diffeomorphic structure, while ResNets are not constrained by such a structure. For $\varepsilon=10^{-3}$, by sampling training points uniformly from $K=[-2, 2]^2$ and using the Adam optimizer \citep{Adam}, we trained a ResNet by minimizing the empirical mean squared error loss. This trained ResNet $R:\mathbb{R}^2\to \mathbb{R}^2$ is of the form
\begin{align}\label{eq_Hwang_ResNet}
	R=T_{W_\beta, B_\beta}\circ \mathcal{B}_{5}\circ \cdots \circ \mathcal{B}_{1} \circ T_{W_\alpha, B_\alpha}, \enspace \mathcal{B}_i=\mathrm{Id}_3 + T_{W_{ib}, B_{ib}}\circ \mathrm{ReLU}\circ T_{W_{ia}, B_{ia}},
\end{align}
and all the parameters are presented in Table \ref{table_appendix_disk_input_output_parameters} and Table \ref{table_appendix_disk_block_parameters} (Appendix \ref{appendix_ResNet_parameter}). The approximation error was evaluated on an independent $801\times801$ uniform Cartesian grid $G\subset[-2, 2]^2$, and the resulting grid-based maximum error 
\begin{align*}
	E_{\infty}^{\mathrm{grid}}:=\underset{x\in G}{\max} \vert\vert R(x)-f(x)\vert\vert_2
\end{align*}
was less than $\varepsilon$.

For the input set $K=[-2, 2]^2$, the images
\begin{align}\label{eq_image_Hwang}
	&\mathrm{INPUT}:= [-2, 2]^2, \enspace \mathrm{OUTPUT}:= R([-2, 2]^2), \nonumber\\
	&A_0 := T_{W_\alpha, B_\alpha} ([-2, 2]^2), \enspace A_i := \mathcal{B}_i\circ\cdots\circ \mathcal{B}_1\circ T_{W_\alpha, B_\alpha}([-2, 2]^2)
\end{align} 
are illustrated in Figure \ref{fig_Hwang}. It can be observed that $A_2, A_3, A_4, A_5$ have a self-intersection structure. Thus, for 
\begin{align*}
	G_i:=B_i\circ\cdots\circ B_1\circ T_{W_\alpha,B_\alpha},
\end{align*}
the corresponding maps $G_i|_K:K\to\mathbb{R}^3$, $i=2,3,4,5$, are numerically observed to be non-injective and hence are not embeddings. In other words, although the local dimension is preserved, global injectivity may fail. In contrast with width-three MLPs, this difference suggests the stronger approximation ability of narrow ResNets than narrow MLPs in some approximation tasks.
\begin{figure}[htbp]
	\centering
	\includegraphics[height=0.45\linewidth, width=0.9\linewidth]{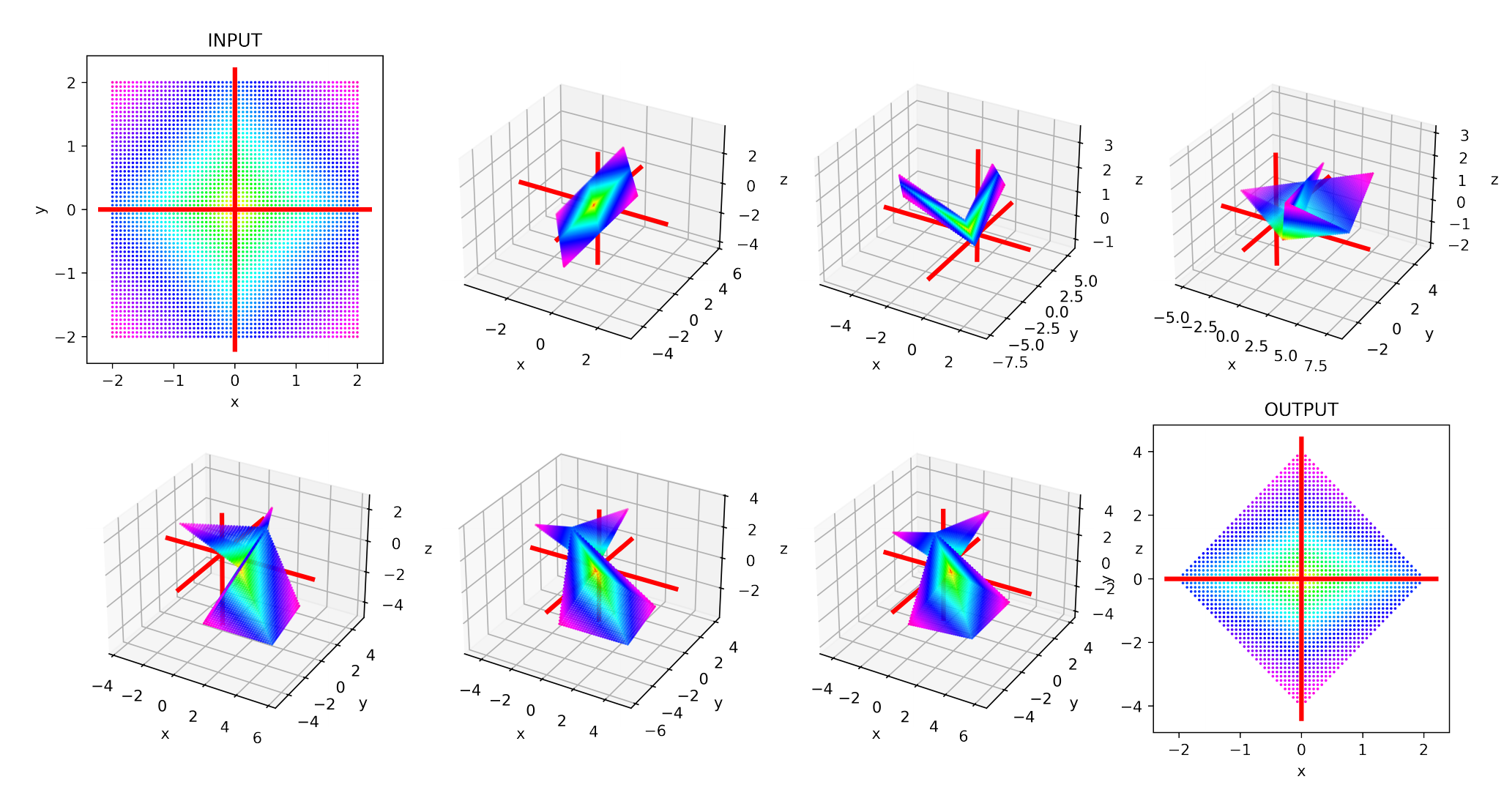}
	\caption{Visualization of the images generated by the ResNet $R:\mathbb{R}^2\to \mathbb{R}^2$ defined in Equation \eqref{eq_Hwang_ResNet} on 
		$K=[-2,2]^2$. The panels are ordered row-wise, from left to right in the top row 
		and then from left to right in the bottom row. They show the input set, the embedded set $A_0$, the successive residual-block 
		images $A_i$, and the final output $R(K)$ defined in Equation \eqref{eq_image_Hwang}.}
	\label{fig_Hwang}
\end{figure}

\section{Discussion}
This paper shows that the ResNets can retain universal approximation power even when each residual branch is restricted to inner width one. Since the inner-width-one structure is the most restrictive structure, all upper bounds established in this paper for the inner-width-one case remain valid for ResNets with larger inner widths. Our results indicate that for ResNets, the block width and the inner width play different roles: the former provides the ambient space for residual transformations, while the latter controls the size of each local nonlinear update.

Besides the block width and the inner width, the depth of ResNets is also a central aspect of their approximation power. Recent work characterizes the role of depth in universal approximation capability of ResNets \citep{Depth_of_ResNet}, which suggests that depth determines how efficiently a target function can be approximated. From this perspective, a natural future direction is to estimate the number of residual blocks required to achieve a prescribed accuracy under the inner-width-one constraint. In addition, increasing the inner width may enhance the approximation power of ResNets. Thus,  whether the required depth can be reduced as the inner width increases is also an interesting question for future research.


\section{Appendices}
\subsection{Activation Function Formulae}\label{appendix_activationfunc}
	The common activation functions are defined in Table \ref{table_activation_functions} \citep{petersen_DL, Kim}.
	\begin{table}[!htbp]
		\clearpage
		\caption{Common activation functions used or mentioned in this paper.}
		\label{table_activation_functions}
		\centering
		\scriptsize
		\setlength{\tabcolsep}{3pt}
		\renewcommand{\arraystretch}{1.75}
		\begin{tabular}{
				|C{0.18\textwidth}
				C{0.18\textwidth}
				C{0.56\textwidth}|}
			\hline
			\textbf{Activation function}
			&
			\textbf{Parameter range}
			&
			\textbf{Formula}
			\\
			\hline
			
			ReLU
			&
			---
			&
			\(\displaystyle
			\mathrm{ReLU}(x)=
			\begin{cases}
				x, & x\geq 0,\\
				0, & x<0.
			\end{cases}
			\)
			\\
			\hline
			
			LeakyReLU
			&
			\(\beta\in(0,1)\cup(1,+\infty)\)
			&
			\(\displaystyle
			\mathrm{LeakyReLU}_{\beta}(x)=
			\begin{cases}
				x, & x\geq 0,\\
				\beta x, & x<0.
			\end{cases}
			\)
			\\
			\hline
			
			ABS
			&
			---
			&
			\(\displaystyle
			\mathrm{ABS}(x)=
			\begin{cases}
				x, & x\geq 0,\\
				-x, & x<0.
			\end{cases}
			\)
			\\
			\hline
			
			Sigmoid
			&
			---
			&
			\(\displaystyle
			\mathrm{Sigmoid}(x)=\frac{1}{1+e^{-x}}
			\)
			\\
			\hline
			
			Tanh
			&
			---
			&
			\(\displaystyle
			\mathrm{Tanh}(x)=\frac{e^x-e^{-x}}{e^x+e^{-x}}
			\)
			\\
			\hline
			
			ELU
			&
			\(\beta\in(0,+\infty)\)
			&
			\(\displaystyle
			\mathrm{ELU}_{\beta}(x)=
			\begin{cases}
				x, & x\geq 0,\\
				\beta(e^x-1), & x<0.
			\end{cases}
			\)
			\\
			\hline
			
			CELU
			&
			\(\beta\in(0,+\infty)\)
			&
			\(\displaystyle
			\mathrm{CELU}_{\beta}(x)=
			\begin{cases}
				x, & x\geq 0,\\
				\beta(e^{x/\beta}-1), & x<0.
			\end{cases}
			\)
			\\
			\hline
			
			SELU
			&
			\(\lambda,\beta\in(0,+\infty)\)
			&
			\(\displaystyle
			\mathrm{SELU}_{\lambda,\beta}(x)=
			\lambda
			\begin{cases}
				x, & x\geq 0,\\
				\beta(e^x-1), & x<0.
			\end{cases}
			\)
			\\
			\hline
			
			Softplus
			&
			\(\beta\in(0,+\infty)\)
			&
			\(\displaystyle
			\mathrm{Softplus}_{\beta}(x)
			=
			\frac{1}{\beta}\log(1+e^{\beta x})
			\)
			\\
			\hline
			
			GELU
			&
			---
			&
			\(\displaystyle
			\mathrm{GELU}(x)=x\Psi(x),\quad
			\Psi(x)=\frac{1}{\sqrt{2\pi}}
			\int_{-\infty}^{x}e^{-t^2/2}\,dt
			\)
			\\
			\hline
			
			SiLU
			&
			\(\beta\in(0,+\infty)\)
			&
			\(\displaystyle
			\mathrm{SiLU}_{\beta}(x)
			=
			x\,\mathrm{Sigmoid}(\beta x)
			\)
			\\
			\hline
			
			Mish
			&
			---
			&
			\(\displaystyle
			\mathrm{Mish}(x)
			=
			x\,\mathrm{Tanh}\!\left(\mathrm{Softplus}_{1}(x)\right)
			\)
			\\
			\hline
			
			ReLU6
			&
			---
			&
			\(\displaystyle
			\mathrm{ReLU6}(x)=
			\begin{cases}
				6, & x>6,\\
				x, & 0\leq x\leq 6,\\
				0, & x<0.
			\end{cases}
			\)
			\\
			\hline
			
			Softshrink
			&
			\(\beta\in(0,+\infty)\)
			&
			\(\displaystyle
			\mathrm{Softshrink}_{\beta}(x)=
			\begin{cases}
				x-\beta, & x>\beta,\\
				0, & -\beta\leq x\leq \beta,\\
				x+\beta, & x<-\beta.
			\end{cases}
			\)
			\\
			\hline
			
			HardSigmoid
			&
			---
			&
			\(\displaystyle
			\mathrm{HardSigmoid}(x)=
			\begin{cases}
				1, & x>3,\\
				\dfrac{x+3}{6}, & -3\leq x\leq 3,\\
				0, & x<-3.
			\end{cases}
			\)
			\\
			\hline
			
			HardTanh
			&
			---
			&
			\(\displaystyle
			\mathrm{HardTanh}(x)=
			\begin{cases}
				1, & x>1,\\
				x, & -1\leq x\leq 1,\\
				-1, & x<-1.
			\end{cases}
			\)
			\\
			\hline
			
			HardSwish
			&
			---
			&
			\(\displaystyle
			\mathrm{HardSwish}(x)=
			\begin{cases}
				x, & x>3,\\
				\dfrac{x(x+3)}{6}, & -3\leq x\leq 3,\\
				0, & x<-3.
			\end{cases}
			\)
			\\
			\hline
		\end{tabular}
	\end{table}

	\subsection{Proofs}
	\subsubsection{Proof of Proposition \ref{prop_appro_affine}}\label{appendix_proof_appro_affine}
	Since $\mathcal{G}\underset{\mathrm{sup}}{\prec}\mathcal{F}$ implies that $\mathcal{G}\underset{p}{\prec}\mathcal{F}$, it suffices to prove the $\mathrm{sup}$ norm case. We prove this proposition by decomposing map, and then utilize Proposition \ref{prop_composition_classes}.
	
	By Definition \ref{def_activationfunc}, $\sigma_0$ is a piecewise $C^1$ function and there exists $\alpha\in \mathbb{R}$ that $\sigma^\prime_0(\alpha)\neq 0$. For any $\varepsilon>0$, we can choose an interval $(q_1, q_2)$ such that
	\begin{align*}
		\left\vert \frac{\sigma_0(x)- \sigma_0(\alpha)}{x-\alpha}- a\right\vert <\varepsilon, \enspace \forall x\in (q_1, q_2).
	\end{align*}
	It follows that
	\begin{align*}
		\vert\vert \sigma_0(x)-(ax+c) \vert\vert_{\mathrm{sup}, (q_1, q_2)} < (q_2-q_1)\varepsilon, \enspace a = \sigma_0^\prime(\alpha), \enspace c = \sigma_0(\alpha)-\alpha * \sigma_0^\prime(\alpha).
	\end{align*}
	Let $K\subset\mathbb{R}^w$ be a compact set, we can choose $k_1< k_2$ such that $K\subset [k_1, k_2]^w$. Choose $r_1\in\mathbb{R}_+, s_1\in\mathbb{R}$ such that 
	\begin{align}\label{eq_subset}
		(\frac{3}{4}q_1+ \frac{1}{4}q_2, \frac{1}{4}q_1+ \frac{3}{4}q_2) \subset[r_1k_1+s_1, r_1k_2+s_1] \subset (q_1, q_2).
	\end{align}
	We first prove a basic approximation step. Fix $d\in\mathbb R$ and $i,j\in{1,\ldots,w}$, and consider the affine transformation $T_{W, B}$ with
	\begin{align*}
		W = I_w + dE_{ij}, \enspace B=
		\begin{pmatrix}
			b_1 \\ b_2\\ \vdots \\ b_w
		\end{pmatrix}.
	\end{align*}
	Define the affine maps in the residual branch by
	\begin{align*}
		&W_1 = r_1 * E_{1\times w, 1, j}\in \mathbb{R}^{1\times w}, \enspace
		B_1 = s_1,\\
		&W_2 = \frac{d}{ar_1} * E_{w\times 1, i, 1}\in \mathbb{R}^{w\times 1}, \enspace
		B_2 = \begin{pmatrix}
			b_1 \\
			\vdots\\
			b_i - \frac{d(as_1 + c)}{ar_1}\\
			\vdots\\
			b_w\\
		\end{pmatrix},\\
		& \mathcal{B} = \mathrm{Id}_w + T_{W_2, B_2} \circ \sigma_0 \circ T_{W_1, B_1},
	\end{align*}
	then $\vert\vert \mathcal{B} - T_{W, B}\vert\vert_{\mathrm{sup}, [k_1, k_2]^w}< \frac{\vert d\vert(q_2-q_1)}{\vert a\vert r_1}\varepsilon$. By the Equation \eqref{eq_subset}, we have $r_1(k_2-k_1)\geq\frac{1}{2}(q_2-q_1)$, so we can conclude that $\vert\vert \mathcal{B} - T_{W, B}\vert\vert_{\mathrm{sup}, [k_1, k_2]^w}< \frac{2(k_2-k_1)\vert d\vert}{\vert a\vert}\varepsilon$, i.e., $T_{W, B} \underset{\mathrm{sup}}{\prec}\mathrm{CRB}_{w, 1}^{\sigma_0}$.
	
	For the case (a), since 
	\begin{align*}
		W&= 
		\scriptsize
		\begin{pmatrix}
			W_{1, 1} & 0 & \vdots & 0\\
			0 & W_{2, 2} & \vdots & 0\\
			\cdots & \cdots & \vdots & \cdots\\
			0 & 0 & \cdots & W_{w, w}
		\end{pmatrix}\\
		&=
		(I_w + (W_{1, 1}-1)E_{w\times w, 1, 1}) * \cdots * (I_w + (W_{w, w}-1)E_{w\times w, w, w}),
	\end{align*}
	thus for any $B\in \mathbb{R}^{w}$, $T_{W, B}=T_{I_w, B}\circ T_{W, \mathbf{0}_{w}}$ can be approximated by $\mathrm{CRB}_{w, 1}^{\sigma_0}$.
	
	For the case (b), we assume that $W$ is an upper triangular matrix whose diagonal entries are nonzero, then the proof for the case that $W$ is a lower triangular matrix is similar. Since we have
	\begin{align*}
		W&\scriptsize
		=\begin{pmatrix}
			W_{1, 1} & W_{1, 2} & W_{1, 3}& \cdots & W_{1, w-1}& W_{1, w}\\
			0 & W_{2, 2} & W_{2, 3}& \cdots & W_{2, w-1} & W_{2, w}\\
			\vdots & \vdots & \vdots &\vdots & \vdots &\vdots \\
			\cdots & \cdots & \cdots & W_{w-2, w-2} & W_{w-2, w-1}& W_{w-2, w}\\
			0 & 0 & 0 &\cdots & W_{w-1, w-1}& W_{w-1, w}\\
			0 & 0 & 0 &\cdots & 0 & W_{w, w}
		\end{pmatrix}\\
		&\scriptsize
		=\begin{pmatrix}
			W_{1, 1} & 0 & 0&\cdots & 0 & 0\\
			0 & W_{2, 2} & 0 &\cdots & 0 & 0\\
			0 & 0 & W_{3, 3} & \cdots &\cdots & 0 \\
			0 & 0 & 0 & \vdots &\cdots & \cdots\\
			0 & 0 & 0 & \cdots & W_{w-1, w-1}& 0\\
			0 & 0 & 0 & \cdots & 0 & W_{w, w}
		\end{pmatrix}
		*\begin{pmatrix}
			1 & 0 & 0 & \cdots & 0 & 0\\
			0 & 1 & 0 &\cdots & 0 & 0\\
			\vdots & \vdots &\vdots&\vdots & \vdots & \vdots\\
			\vdots & \vdots &\vdots&\vdots & \vdots & \vdots\\
			0 & 0 & \cdots&\cdots & 1 & \frac{W_{w-1, w}}{W_{w-1, w-1}}\\
			0 & 0 & \cdots&\cdots & 0 & 1
		\end{pmatrix}\\
		&\scriptsize*\cdots 
		*
		\begin{pmatrix}
			1 & 0 & 0& \cdots & 0& 0\\
			0 & 1 & \frac{W_{2, 3}}{W_{2, 2}}& \cdots & \frac{W_{2, w-1}}{W_{2, 2}}& \frac{W_{2, w}}{W_{2, 2}}\\
			\vdots & \vdots & \vdots & \vdots &\vdots & \vdots\\
			\vdots & \vdots & \vdots & \vdots &\vdots & \vdots\\
			0 & 0 & \cdots&\cdots & 1 & 0\\
			0 & 0 & \cdots&\cdots & 0 & 1
		\end{pmatrix}
		*
		\begin{pmatrix}
			1 & \frac{W_{1, 2}}{W_{1, 1}} & \frac{W_{1, 3}}{W_{1, 1}}& \cdots & \frac{W_{1, w-1}}{W_{1, 1}}& \frac{W_{1, w}}{W_{1, 1}}\\
			0 & 1 & 0 &\cdots & 0 & 0\\
			\vdots & \vdots & \vdots & \vdots &\vdots & \vdots\\
			\vdots & \vdots & \vdots & \vdots &\vdots & \vdots\\
			0 & 0 & \cdots&\cdots & 1 & 0\\
			0 & 0 & \cdots&\cdots & 0 & 1
		\end{pmatrix},
	\end{align*}
	it is evident that $T_{W, B}=T_{I_w, B}\circ T_{W, \mathbf{0}_{w}}$ can be approximated by $\mathrm{CRB}_{w, 1}^{\sigma_0}$.
	
	For the case (c), since the row-swap permutation matrix swapping rows $i$ and $j$ $(i<j)$ satisfies
	\begin{align*}
		W=(I_w+E_{w\times w, j, i}) * (I_w-E_{w\times w, i, j}) * (I_w - 2E_{w\times w, j, j}) * (I_w - E_{w\times w, j,i}),
	\end{align*}
	and every permutation matrix can be written as a product of such matrices, we can show that for every permutation matrix $W$ and $B\in \mathbb{R}^w$, $T_{W, B}=T_{I_w, B}\circ T_{W, \mathbf{0}_w}$ can be approximated by $\mathrm{CRB}_{w, 1}^{\sigma_0}$.
	\subsubsection{Proof of Proposition \ref{prop_ReLU-like}}\label{appendix_proof_ReLU-like}
	\begin{itemize}
		\item $\sigma=\mathrm{ELU}$, $\mathrm{CELU}$, $\mathrm{SELU}$:\\
		Fix a compact interval $K=[k_1, k_2]$. If $x\in K$ and $x>0$, for any $k, n\in \mathbb{R}_+$, we have $\frac{\mathrm{ELU}_k (nx)}{n}=x$. If $x\in K$ and $x\leq 0$, it is easy to show that 
		\begin{align*}
			\lim_{n\to +\infty} \frac{\mathrm{ELU}_k (nx)}{n} &= \lim_{n\to +\infty} \frac{ke^{nx}-k}{n}=0,
		\end{align*} 
		so for any $\varepsilon>0$, we can choose a sufficiently large $n\in\mathbb{N}_+$ such that
		\begin{align*}
			\vert\vert T_{(\frac{1}{n}), (0)}\circ \sigma_i \circ T_{(n), (0)} - \mathrm{ReLU}\vert\vert_{\mathrm{sup}, K}<\varepsilon,
		\end{align*}
		then we can conclude that $\mathrm{ELU}$ is ReLU-like. The proofs of the cases that $\sigma=\mathrm{CELU}, \mathrm{SELU}$ are handled similarly.
		\item $\sigma=\mathrm{Softplus}, \mathrm{GELU}, \mathrm{SiLU}, \mathrm{Mish}$:\\
		These cases follow from the proof of the Lemma 23 of \citep{Kim}.
		
		\item $\sigma=\mathrm{ReLU6}, \mathrm{Softshrink}, \mathrm{HardSigmoid}, \mathrm{HardTanh}$:\\
		For the fixed compact set $K=[k_1, k_2]$, we first prove the case that $\sigma=\mathrm{ReLU6}$. If $k_1\geq0$, choose $a\in\mathbb{R}_+, b\in \mathbb{R}$ that $[ak_1+b, ak_2+b]\subset [0, 6]$, then
		\begin{align*}
			T_{(\frac{1}{a}), (-\frac{b}{a})} \circ \sigma \circ T_{(a), (b)}(x) -\mathrm{ReLU}(x)=0, \enspace x\in K.
		\end{align*}
		If $k_2\leq0$, then 
		\begin{align*}
			T_{(1), (0)} \circ \sigma \circ T_{(1), (0)}(x) -\mathrm{ReLU}(x)=0, \enspace x\in K.
		\end{align*}
		If $k_1<0< k_2$, there exists $a\in \mathbb{R}_+$ such that $a<\frac{6}{k_2}$, then 
		\begin{align*}
			T_{(\frac{1}{a}), (0)} \circ \sigma \circ T_{(a), (0)}(x) - \mathrm{ReLU}(x) = 0, \enspace x\in K,
		\end{align*}
		which completes the proof for the case of $\sigma=\mathrm{ReLU6}$. The proofs of the cases that $\sigma=\mathrm{Softshrink}, \mathrm{HardSigmoid}, \mathrm{HardTanh}$ are similar to the case of $\sigma=\mathrm{ReLU6}$.
		\item $\sigma=\mathrm{HardSwish}$:\\
		For any compact set $K$, we can find $[-k_1, k_1]$ such that $K\subset[-k_1, k_1], k_1>0$. Since for $n>\frac{3}{k_1}$, it is easy to verify that 
		\begin{align*}
			\vert\vert T_{(\frac{1}{n}), (0)}\circ \sigma \circ T_{(n), (0)} - \mathrm{ReLU}\vert\vert_{\mathrm{sup}, [-k_1, k_1]}<\frac{3}{n},
		\end{align*} 
		so for any $\varepsilon>0$, we can choose a sufficiently large $n\in \mathbb{N}_+$ such that 
		\begin{align*}
			\vert\vert T_{(\frac{1}{n}), (0)}\circ \sigma \circ T_{(n), (0)} - \mathrm{ReLU}\vert\vert_{\mathrm{sup}, [-k_1, k_1]}<\varepsilon,
		\end{align*} 
		which completes the proof.
	\end{itemize}	
	
	\subsubsection{Proof of Theorem \ref{thm_slopeline}}\label{appendix_proof_slopeline}
	We first explain how a scalar construction can be lifted coordinate-wise to $\mathbb{R}^w$. Let $s:\mathbb{R}\to\mathbb{R}$ that acts only on the $i$-th coordinate $x\to (x_1,\cdots, s(x_i), \cdots,x_w)$, then the one dimensional map can be lifted to the coordinate-wise map $x\to (s(x_1),\ldots,s(x_w))$.
	
	Thus, in the following proof, it is enough to construct the desired scalar map in one dimension, then we apply the same construction successively to each coordinate gives the corresponding map on $\mathbb{R}^w$.
	
	We first suppose that $\sigma=\mathrm{LeakyReLU}_k$ for some fixed $k\in (0, 1)\cup (1, +\infty)$, then
	$$
	x+b\mathrm{LeakyReLU}_k(x)=
	\left\{
	\begin{aligned}
		(1+b)x,& \enspace x> 0,\\
		(1+kb)x, & \enspace x\leq0,
	\end{aligned}
	\right.
	$$	
	For $b\neq -1$, the ratio $\frac{1+kb}{1+b}=k+\frac{1-k}{1+b}$ ranges over $(-\infty, k)\cup(k, +\infty)$. 
	
	Therefore, for $a\neq k$, we can show that
	\begin{align*}
		f_a &=\frac{1}{1+b}\mathrm{Id}_w\circ(\mathrm{Id}_w + T_{b*E_{w\times 1, 1, 1}, \mathbf{0}_{w\times 1}}\circ \mathrm{LeakyReLU}_k \circ T_{E_{1\times w, 1, 1}, 0})\\
		&\circ (\mathrm{Id}_w + T_{b*E_{w\times 1, 2, 1}, \mathbf{0}_{w\times 1}}\circ \mathrm{LeakyReLU}_k \circ T_{E_{1\times w, 1, 2}, 0})
		\circ \cdots \\
		&\circ (\mathrm{Id}_w + T_{b*E_{w\times 1, w, 1}, \mathbf{0}_{w\times 1}}\circ \mathrm{LeakyReLU}_k \circ T_{E_{1\times w, 1, w}, 0}).
	\end{align*}
	By employing Theorem \ref{thm_appro_affine}, $\frac{1}{1+b}\mathrm{Id}_w \underset{\mathrm{sup}}{\prec} \mathrm{CRB}_{w, 1}^{\mathrm{LeakyReLU}_k},$
	thus we have $f_a\underset{\mathrm{sup}}{\prec} \mathrm{CRB}^{\mathrm{LeakyReLU}_k}_{w, 1}$.
	
	For $a=k$, by Proposition \ref{prop_LRk_to_LR}, it follows that $f_a = \mathrm{LeakyReLU}_a \underset{\mathrm{sup}}{\prec} \mathrm{CRB}_{w, 1}^{\mathrm{LeakyReLU}_k}.$
	
	Next, we suppose that $\sigma={\mathrm{ReLU}}$, since for any $b\in \mathbb{R}$,
	\begin{align*}
		x+b\mathrm{ReLU}(x)=\left\{
		\begin{aligned}
			(1+b)x&, \enspace x>0,\\
			x&, \enspace x\leq 0,
		\end{aligned}
		\right. ,
	\end{align*}
	then factor $\frac{1}{1+b}\in (-\infty, 0)\cup (0, +\infty)$. For $a\neq 0$, choose $b=\frac{1}{a}-1$. Then
	\begin{align*}
		f_a =\frac{1}{1+b}\mathrm{Id}_w \circ (\mathrm{Id}_w + T_{b*E_{w\times 1, 1, 1}, \mathbf{0}_{w\times 1}}\circ \mathrm{ReLU} \circ T_{E_{1\times w, 1, 1}, 0})
		\circ\\ (\mathrm{Id}_w + T_{b*E_{w\times 1, 2, 1}, \mathbf{0}_{w\times 1}}\circ \mathrm{ReLU} \circ T_{E_{1\times w, 1, 2}, 0})
		\circ\\
		\cdots
		\circ (\mathrm{Id}_w + T_{b*E_{w\times 1, w, 1}, \mathbf{0}_{w\times 1}}\circ \mathrm{ReLU} \circ T_{E_{1\times w, 1, w}, 0}),
	\end{align*}
	so we can conclude that $f_a\underset{\mathrm{sup}}{\prec} \mathrm{CRB}^{\mathrm{ReLU}}_{w, 1}$
	by employing Theorem \ref{thm_appro_affine}. For $a=0$, since for any compact set $K$ and $\varepsilon>0$, we can find $a>0$ such that $\vert\vert f_a-f_0\vert\vert_{\mathrm{sup}, K}<\varepsilon,$
	so we can conclude that $f_0\underset{\mathrm{sup}}{\prec} \{f_a\vert a\in \mathbb{R}_+\} \underset{\mathrm{sup}}{\prec} \mathrm{CRB}^{\mathrm{ReLU}}_{w, 1}.$
	
	Finally, suppose that $\sigma$ is ReLU-like. Let $K\subset\mathbb{R}^w$ be compact and choose $k_1<k_2$ such that $K\subset [k_1, k_2]^w$. By Definition \ref{def_ReLU-like}, for any $\varepsilon>0$ and any fixed parameter $b\neq-1$, there exist $a_1, a_2, b_1, b_2\in\mathbb{R}$ such that 
	\begin{align*}
		&\vert\vert \frac{1}{1+b}\mathrm{Id}_w \circ (\mathrm{Id}_w+T_{a_2E_{w\times 1, 1, 1}, b_2E_{w\times 1,1 , 1}} \circ \sigma \circ T_{a_1E_{1\times w, 1, 1}, b_1})\circ \cdots \circ \\
		&(\mathrm{Id}_w+T_{a_2E_{w\times 1, w, 1}, b_2E_{w\times 1, w, 1}} \circ \sigma \circ T_{a_1E_{1\times w, 1, w}, b_1})  - f_{\frac{1}{1+b}}\vert\vert_{\mathrm{sup}, K}< \varepsilon,
	\end{align*}
	Thus the ReLU case carries over by replacing each ReLU block with its ReLU-like approximation.
	
	\subsection{Parameter Settings of Networks}\label{appendix_ResNet_parameter}
	The parameters for the network in Figure \ref{fig_Polyline} and Figure \ref{fig_Hwang} are presented. For Figure \ref{fig_Hwang}, elements in matrices are retained to six decimal places when their absolute values exceed $10^{-6}$. All the parameters are listed in the following tables.
		\begin{table}[!htbp] 
		\caption{Parameters of the initial and final affine transformations.} 
		\centering 
		\scriptsize
		\setlength{\tabcolsep}{4pt} 
		\renewcommand{\arraystretch}{1.65} 
		\begin{tabular}
			{ @{} C{0.20\textwidth} C{0.34\textwidth} C{0.34\textwidth} @{}} 
			\toprule Layer & Weight matrix & Bias vector\\ 
			\midrule 
			Input affine map & $ W_\alpha= \begin{pmatrix}
				10 \\ 0  
			\end{pmatrix} $ & $ B_\alpha= \begin{pmatrix} 
				-1\\ 0 
			\end{pmatrix} $ \\
			\midrule 
			Output affine map & $ W_\beta= \begin{pmatrix}
				1 & 0 \\
				0 & 1
			\end{pmatrix} $ & $ B_\beta= \begin{pmatrix}
				0\\ 0
			\end{pmatrix} $ \\
			\bottomrule
		\end{tabular} 
		\label{table_appendix_polyline_input_output_parameters} 
	\end{table}
	\begin{table}[!htbp] 
		\caption{Parameters of the residual blocks.} 
		\centering 
		\scriptsize
		\setlength{\tabcolsep}{3pt} 
		\renewcommand{\arraystretch}{1.75} 
		\begin{tabular}
			{ @{} C{0.08\textwidth} C{0.36\textwidth} C{0.12\textwidth} C{0.20\textwidth} C{0.20\textwidth} @{}}
			\toprule 
			Block & $W_{ia}$ & $B_{ia}$ & $W_{ib}$ & $B_{ib}$\\
			\midrule
			$i=1$ & $ \begin{pmatrix}
				1 & 0 
			\end{pmatrix} $ & $ \begin{pmatrix} 
				-2 
			\end{pmatrix} $ & $ \begin{pmatrix}
				-1\\ 1 
			\end{pmatrix} $ & $ \begin{pmatrix}
				0\\ 0 
			\end{pmatrix} $ \\
			\midrule
			$i=2$ & $
			\begin{pmatrix}
				0 & 1  
			\end{pmatrix} $ & $ \begin{pmatrix}
				-2 
			\end{pmatrix} $ & $ \begin{pmatrix}
				-1\\ -1
			\end{pmatrix} $ & $ \begin{pmatrix}
				0\\ 0
			\end{pmatrix} $ \\
			\midrule
			$i=3$ & $ \begin{pmatrix}
				-1 & 1
			\end{pmatrix} $ & $ \begin{pmatrix}
				-2
			\end{pmatrix} $ & $ \begin{pmatrix}
				1\\ -1
			\end{pmatrix} $ & $ \begin{pmatrix}
				0\\ 0
			\end{pmatrix} $ \\
			\bottomrule
		\end{tabular}
		\label{table_appendix_polyline_block_parameters} 
	\end{table}

\begin{table}[!htbp]
	\caption{Parameters of the initial and final affine transformations.}
	\centering
	\scriptsize
	\setlength{\tabcolsep}{4pt}
	\renewcommand{\arraystretch}{1.65}
	\begin{tabular}
		{@{} C{0.20\textwidth} C{0.34\textwidth} C{0.34\textwidth} @{}}
		\toprule
		Layer & Weight matrix & Bias vector\\
		\midrule
		Input affine map
		&
		$W_\alpha=
		\begin{pmatrix}
			-0.779449 & -0.624719\\
			1.786707 & -0.402260\\
			-0.353954 & -1.217714
		\end{pmatrix}$
		&
		$B_\alpha=
		\begin{pmatrix}
			-0.270580\\
			0.309483\\
			-0.331742
		\end{pmatrix}$
		\\
		\midrule
		Output affine map
		&
		$W_\beta=
		\begin{pmatrix}
			-0.078566 &  0.284202 & -0.575851\\
			-1.070493 & -0.740750 & -0.383484
		\end{pmatrix}$
		&
		$B_\beta=
		\begin{pmatrix}
			0.758574\\
			0.799097
		\end{pmatrix}$
		\\
		\bottomrule
	\end{tabular}
	\label{table_appendix_disk_input_output_parameters}
\end{table}

\begin{table}[!htbp]
	\caption{Parameters of the residual blocks.}
	\centering
	\scriptsize
	\setlength{\tabcolsep}{3pt}
	\renewcommand{\arraystretch}{1.75}
	\begin{tabular}
		{@{} C{0.08\textwidth} C{0.36\textwidth} C{0.12\textwidth}
			C{0.20\textwidth} C{0.20\textwidth} @{}}
		\toprule
		Block & $W_{ia}$ & $B_{ia}$ & $W_{ib}$ & $B_{ib}$\\
		\midrule
		$i=1$
		&
		$\begin{pmatrix}
			-0.723092 & -0.534896 & -1.107739
		\end{pmatrix}$
		&
		$\begin{pmatrix}
			-0.397597
		\end{pmatrix}$
		&
		$\begin{pmatrix}
			-0.354179\\
			-0.538340\\
			1.114135
		\end{pmatrix}$
		&
		$\begin{pmatrix}
			-0.234547\\
			-0.105348\\
			0.255320
		\end{pmatrix}$
		\\
		\midrule
		$i=2$
		&
		$\begin{pmatrix}
			-0.337070 & -1.026687 & -0.603239
		\end{pmatrix}$
		&
		$\begin{pmatrix}
			-0.006781
		\end{pmatrix}$
		&
		$\begin{pmatrix}
			1.445647\\
			1.391393\\
			0.139565
		\end{pmatrix}$
		&
		$\begin{pmatrix}
			0.793295\\
			0.111727\\
			-0.807420
		\end{pmatrix}$
		\\
		\midrule
		$i=3$
		&
		$\begin{pmatrix}
			0.892487 & -0.911565 & -0.156742
		\end{pmatrix}$
		&
		$\begin{pmatrix}
			-0.107793
		\end{pmatrix}$
		&
		$\begin{pmatrix}
			-0.615240\\
			-0.283921\\
			-1.029730
		\end{pmatrix}$
		&
		$\begin{pmatrix}
			0.926004\\
			-0.445451\\
			0.386327
		\end{pmatrix}$
		\\
		\midrule
		$i=4$
		&
		$\begin{pmatrix}
			0.218183 & -0.309212 & 0.822697
		\end{pmatrix}$
		&
		$\begin{pmatrix}
			0.104323
		\end{pmatrix}$
		&
		$\begin{pmatrix}
			-1.179464\\
			-0.700994\\
			0.049330
		\end{pmatrix}$
		&
		$\begin{pmatrix}
			-0.841900\\
			-0.999599\\
			1.042684
		\end{pmatrix}$
		\\
		\midrule
		$i=5$
		&
		$\begin{pmatrix}
			0.352745 & -0.053745 & 0.608389
		\end{pmatrix}$
		&
		$\begin{pmatrix}
			0.649316
		\end{pmatrix}$
		&
		$\begin{pmatrix}
			0.563700\\
			-0.246493\\
			0.474222
		\end{pmatrix}$
		&
		$\begin{pmatrix}
			0.020805\\
			0.580137\\
			-0.341635
		\end{pmatrix}$
		\\
		\bottomrule
	\end{tabular}
	\label{table_appendix_disk_block_parameters}
\end{table}
	
\clearpage
\section*{Declarations}
\textbf{Acknowledgments}: We thank the School of Mathematics and Statistics, Huazhong University of Science and Technology (Wuhan, 430074, China), and the Hubei Key Laboratory of Engineering Modeling and Scientific Computing (Wuhan, 430074, China) for their support.

\textbf{Funding}: This work was supported by the National Natural Science Foundation of China (Grant No. 12531017). The funder had no role in the study design, analysis or interpretation of the results, preparation of the manuscript, or the decision to submit the work for publication.

\textbf{CRediT authorship contribution statement}:

Qi Zhou: Conceptualization, methodology, formal analysis,
investigation, software, visualization, writing -- original draft.

Xuan Zhou: Conceptualization, methodology, formal analysis,
validation, writing -- review \& editing.

Xiao-Song Yang: Conceptualization, supervision, project
administration, funding acquisition, writing -- review \& editing.

\textbf{Declaration of competing interest}:
The authors declare that they have no known competing financial
interests or personal relationships that could have appeared to
influence the work reported in this paper.

\textbf{Data and code availability}:
No external datasets were used in this study. All network parameters
used in the numerical case studies are provided in the manuscript and
its appendices. The source code used to generate the numerical results
and figures is available from the corresponding author upon reasonable
request.

\textbf{Declaration of generative AI and AI-assisted technologies in
	the manuscript preparation process}:
During the preparation of this work, the authors used ChatGPT and Codex
(OpenAI) to assist with language editing and manuscript organization. After using these
tools, the authors independently reviewed, verified, and revised all
mathematical arguments, numerical results, references, and text as
needed and take full responsibility for the content of the published
article.
\bibliographystyle{elsarticle-harv} 
\bibliography{ReferenceFile_ResNet_Neural_Networks}





\end{document}